\documentclass[journal]{IEEEtran}
\IEEEoverridecommandlockouts
\usepackage{cite}
\usepackage{amsmath,amssymb,amsfonts}
\usepackage{algorithmic}
\usepackage{graphicx}
\usepackage{textcomp}
\usepackage{hyperref}
\usepackage{epsfig, psfrag, epstopdf, amstext,amsfonts,amssymb,subfigure,subeqnarray,nccfoots,color,multirow,bm}
\usepackage[linesnumbered,ruled]{algorithm2e}
\usepackage{amsthm}
\newtheorem{theorem}{Theorem}
\newtheorem{lemma}{Lemma}
\newtheorem{remark}{Remark}

\newtheorem{corollary}{Corollary}
\def\BibTeX{{\rm B\kern-.05em{\sc i\kern-.025em b}\kern-.08em
    T\kern-.1667em\lower.7ex\hbox{E}\kern-.125emX}}
\begin{document}

\title{Linear Multiple Low-Rank Kernel Based Stationary Gaussian Processes Regression for Time Series}

\author{\IEEEauthorblockN{Feng Yin, Lishuo Pan, Xinwei He, Tianshi Chen, Sergios Theodoridis, Zhi-Quan (Tom) Luo\\} \thanks{The conference version of this paper \cite{Yin18} has been published in the proceedings of 21st International Conference on Information Fusion (FUSION), University of Cambridge, Cambridge, UK, July, 2018.}
	\IEEEauthorblockA{School of Science and Engineering, The Chinese University of Hong Kong, Shenzhen\\ Shenzhen Research Institute of Big Data (SRIBD) \\  Longxiang blvd 2001, Longgang district, Shenzhen, China, 518172.}}

\maketitle

\begin{abstract}
Gaussian processes (GP) for machine learning have been studied systematically over the past two decades and they are by now widely used in a number of diverse applications. However, GP kernel design and the associated hyper-parameter optimization are still hard and to a large extend open problems. In this paper, we consider the task of GP regression for time series modeling and analysis. The underlying stationary kernel can be approximated arbitrarily close by a new proposed grid spectral mixture (GSM) kernel, which turns out to be a linear combination of low-rank sub-kernels. In the case where a large number of the sub-kernels are used, either the Nystr\"{o}m or the random Fourier feature approximations can be adopted to deal efficiently with the computational demands. The unknown GP hyper-parameters consist of the non-negative weights of all sub-kernels as well as the noise variance; their estimation is performed via the maximum-likelihood (ML) estimation framework. Two efficient numerical optimization methods for solving the unknown hyper-parameters are derived, including a sequential majorization-minimization (MM) method and a non-linearly constrained alternating direction of multiplier method (ADMM). The MM matches perfectly with the proven low-rank property of the proposed GSM sub-kernels and turns out to be a part of efficiency, stable, and efficient solver, while the ADMM has the potential to generate better local minimum in terms of the test MSE. Experimental results, based on various classic time series data sets, corroborate that the proposed GSM kernel-based GP regression model outperforms several salient competitors of similar kind in terms of prediction mean-squared-error and numerical stability. 
\end{abstract}


\section{Introduction}
\label{sec:Introduction}
Gaussian processes (GP) constitute a class of important Bayesian non-parametric models for machine learning and they are tightly connected to several other popular models, such as support vector machines (SVM), regularized-least-squares, relevance vector machines and auto-regressive-moving-average (ARMA), single-layer Bayesian neural networks \cite{RW06} and, more recently, to the deep neural networks \cite{Lee18, Mattews18}. Gaussian processes are also used as outstanding surrogate functions for Bayesian optimization nowadays \cite{Snoek12}. The idea behind the GP models is to impose a Gaussian prior on the underlying function/system and then compute the posterior distribution over the function given the observed data. GP models have been used in a plethora of applications due to their outstanding performance in function approximation with a natural uncertainty bound. 

Gaussian processes models are simple in terms of mathematical formulation and analysis thanks to the underlying Gaussian assumption. However, like other kernel methods, such as support vector machines, one major problem with GP models lies in the selection of an appropriate kernel function. It is well known that a good kernel function is capable of lifting the raw features to a much higher (even infinite) dimensional space, where regression and classification can be done more effectively, e.g., \cite{Theodoridis15}. In practice, kernel selection is often done subjectively, relying on eye inspection of data patterns and a handful of elementary kernels such as the linear kernel, squared-exponential (SE) kernel, Mat\'{e}rn kernel and their hybrid are popular alternatives. For instance, the SE kernel was used for sport trajectory modeling in \cite{Zhao16} and for financial data modeling and prediction in \cite{Han16}, while linear and Mat\'{e}rn kernels were used for energy load forecasting in \cite{Prakash18}, to mention a few in different sectors, even though the selected kernel may not fit the data well.

In order to bypass the need for human intervention, automatic and optimal kernel design is largely demanded. One option is to resort to multiple kernel learning techniques. Multiple kernel refers to learning a linear or nonlinear combination of primitive kernels systematically for a target machine learning (supervised, un-supervised, etc.) model, via a specific optimization method with the goal to let data determine the best kernel configuration. This idea has been implemented mostly based on linear multiple kernel (LMK) for supervised SVM models \cite{Gonen11, Lanckriet04}, for supervised regularized least-squares models \cite{Aravkin14}, and for un-supervised data clustering \cite{Zhuang11}, etc. The idea of mixing elementary kernels for Gaussian process regression also exists, e.g., for $\textrm{CO}_{2}$ prediction in \cite{RW06} and for other prediction tasks in a few recent works \cite{Senanayake2016, Xu17, Prakash18}. However, the main drawback is that the primitive kernels are selected subjectively and primitively combined with equal weights. In other words, the weights were pre-selected and not learnt via optimizing a performance metric, and the resulting simple equal-weighted linear combination of primitive kernels may be way sub-optimal for fitting the given data.

There also exist some competing universal kernel design methods. In \cite{Gredilla10}, Lazaro-Gredilla \textit{et.al.}~proposed a sparse spectrum Gaussian process (SSGP) that extends the linear trigonometric Bayesian model. The spectral density of a stationary covariance kernel is sparsified to approximate the standard GP. The SSGP learns the model hyper-parameters, including the spectral points, precision of a prior, noise variance as well as the lengthscales of the automatic relevance determination (ARD) kernel via maximizing the marginal likelihood with the conjugate gradient method. In \cite{DLG13}, Duvenaud \textit{et.al.}~defined a space of kernel structures built compositionally by adding and multiplying a small number of primitive kernels and search for the optimal combination over the space. In \cite{WA13, Wilson14}, Wilson \textit{et.al.}~proposed a spectral mixture (SM) kernel with the idea to approximate the spectral density with a Gaussian mixture model first in the frequency domain and transform it back into the time domain.

The predictive performance of GP regression depends on the goodness of the model parameters, often referred to as hyper-parameters. There exist two classes of methods for tuning the GP hyper-parameters. The class of deterministic methods consists of the maximum likelihood (ML) estimation based method and cross-validation (CV) based method among others \cite{RW06, Krauth17}. The class of stochastic methods includes for instance the hybrid Monte-Carlo and Markov chain Monte-Carlo (MCMC) sampling methods \cite{Neal97, Neal13}. In this paper, we follow the deterministic ML based method that is more widely used in the GP community.

The main contributions of this work are the following:
\begin{itemize}
\item Based on the assumption that there exists a true kernel and moreover it is stationary, we propose a novel grid spectral mixture (GSM) kernel for time series modeling and analysis. The GSM kernel simply modifies the original spectral mixture kernel \cite{Wilson14} by fixing the frequency and variance parameters to a set of pre-selected grids while leaving only the weights to be optimized. 
\item As a major contribution, the resulting GSM kernel belongs to the class of linear multiple kernels, and the associated sub-kernels are proven to have low-rank property under reasonable conditions. Moreover, by fixing the grids, the ML based hyper-parameter optimization task becomes equivalent to a difference-of-convex problem with nicer structure to be dealt with. When the proposed GSM kernel contains a large number of sub-kernels, we propose to apply Nystr\"{o}m or random Fourier feature approximation for saving in computational complexity and storage requirements.
\item As another major contribution, we derive two effective numerical methods for tuning the GP hyper-parameters. The first method is a sequential majorization-minimization (MM) method, and the second one is an non-linearly constrained alternating direction of multiplier method (ADMM). The former method turns out to be very fast and stable, while the latter method has the potential to achieve a better local minimum in the sense of achieving smaller prediction mean-squared-error (MSE). For both methods, the solution turns out to be sparse, which is a welcome feature in the context of data overfitting problem.
\item Tests based on eight standard time series data sets in various aspects verify that the proposed GSM kernel for GP modeling, empowered with an efficient hyper-parameter optimization approach, is able to achieve much improved prediction performance and robustness as compared to other competing GP models of similar kind. 
\end{itemize}

The remainder of this paper is organized as follows. Section~\ref{sec:ReviewGP} provides the backgroud about Gaussian process regression, the classic ML based hyper-parameter optimization, and the linear multiple kernel. Section~\ref{sec:GSMKernel} first reviews the SM kernel, followed by a new GSM kernel, which turns out to be a linear multiple kernel. Section~\ref{sec:GSM-Approx} introduces the Nystr\"{o}m and random Fourier feature approximation of the GSM sub-kernels for computational and memory savings. Section~\ref{sec:hyperpara-opt-GSMKernel} first presents the ML based hyper-parameter estimation problem for large scale linear multiple kernel, including the proposed GSM kernel and it further presents two numerical optimization methods, namely a sequential MM method and an ADMM method. Experimental results are given in Section~\ref{sec:results}. Finally, Section~\ref{sec:Conclusion} concludes this paper. Proofs of some important properties of the GSM kernel are given in Appendix. 

\textit{Notation:} Throughout this paper, matrices are presented with boldface uppercase letters, vectors with boldface lowercase letters, and scalars with normal lowercase letters. We use $\mathbb{R}$ to denote the set of real numbers. The operator $\left[\cdot\right]^{T}$ stands for vector/matrix transpose, $\textrm{tr}(\cdot)$ for trace of a square matrix, $\textrm{rank}(\cdot)$ for rank of a matrix, $\parallel \cdot \parallel_{p}$ for $L_p$ norm of a vector and $\parallel \cdot \parallel_{F}$ for the Frobenius norm of a matrix, $\mathbb{E}_{p(x)}(\cdot)$ for the expectation taken with respect to the probability density function (PDF) $p(x)$, $\nabla_{\bm{\theta}}$ for gradient, $\mathcal{N}(v; \mu, \sigma^{2})$ for Gaussian distribution of a random variable $V$ with mean $\mu$ and variance $\sigma^{2}$, $\det(\cdot)$ is determinant of a matrix, $\textrm{erf}(x)$ is Gaussian error function, $[a]_{+}$ takes the maximum between $a$ and zero. Lastly, $\boldsymbol{X} \succeq \boldsymbol{Y}$ means $\boldsymbol{X} - \boldsymbol{Y}$ is positive semi-definite, $\left\langle \boldsymbol{X}, \boldsymbol{Y} \right\rangle$ is the inner product of two square matrices, $\boldsymbol{X} \circ \boldsymbol{Y}$ stands for the Hadamad (point-wise) matrix multiplication of $\boldsymbol{X}$ and $\boldsymbol{Y}$.

\section{Background}
\label{sec:ReviewGP}
In this section, we first review GP regression in subsection~\ref{subsec:GPR} and classic ML based GP hyper-parameter optimization in subsection~\ref{subsec:HyperparaOpt}. Lastly, we introduce linear multiple kernel in subsection~\ref{subsec:MLK}. 
%
%
\subsection{GP Regression}
\label{subsec:GPR}
A Gaussian process is a collection of random variables, any finite subset of which follows a Gaussian distribution \cite{RW06}. In the sequel, we solely focus on scalar output, real-valued Gaussian processes that are completely specified by a mean function and a kernel function (a.k.a. covariance function). Concretely, we express 
\begin{equation}
f(\boldsymbol{x}) \sim \mathcal{GP}(m(\boldsymbol{x}), k(\boldsymbol{x}, \boldsymbol{x}'; \boldsymbol{\theta}_{h})),
\end{equation} 
where $m(\boldsymbol{x}) $ is the mean function, which is often set to zero in practice, especially when there is no prior knowledge available; and $k(\boldsymbol{x}, \boldsymbol{x}'; \boldsymbol{\theta}_{h})$ is the kernel function tuned by the kernel hyper-parameters, $\boldsymbol{\theta}_{h}$.

Let us consider the following GP regression model
\begin{equation}
y = f(\boldsymbol{x}) + e, 
\end{equation}
where $y \in \mathbb{R}$ is a continuous-valued, scalar output; the unknown function $f(\boldsymbol{x}) : \mathbb{R}^{d} \mapsto \mathbb{R}$ is modeled as a zero mean Gaussian process for simplicity; and the noise $e$ is assumed to be Gaussian distributed with zero mean and variance $\sigma_{e}^{2}$. Moreover, the noise terms at different data points are assumed to be mutually independent. The set of all unknown GP hyper-parameters is denoted by $\boldsymbol{\theta} \triangleq [\boldsymbol{\theta}_{h}^{T}, \sigma^{2}_{e}]^T$ and the dimension of $\boldsymbol{\theta}$ is assumed to be $p$.

Given a training data set $\mathcal{D} \triangleq \{\boldsymbol{X}, \boldsymbol{y} \}$, where $\boldsymbol{y} = [y_1, y_2, ..., y_n]^T$ is the vector comprising the outputs and $\boldsymbol{X}=[\boldsymbol{x}_1, \boldsymbol{x}_2,...,\boldsymbol{x}_n]$ is the matrix comprising the input vectors, the aim is to compute the posterior distribution of $\boldsymbol{y}_{*} =  [y_{*,1}, y_{*,2},...,y_{*,n_{*}}]^T$ given the corresponding test inputs $\boldsymbol{X}_{*} = [\boldsymbol{x}_{*,1}, \boldsymbol{x}_{*,2},...,\boldsymbol{x}_{*,n_{*}}]$. Here, we let $\mathcal{D}_{*} \triangleq \{\boldsymbol{X}_{*},  \boldsymbol{y}_{*}\}$ be the test data set. 
According to the definition of Gaussian processes given before, the joint prior distribution of the training output $\boldsymbol{y}$ and test output $\boldsymbol{y}_{*}$ can be written explicitly as:
\begin{equation}
\begin{bmatrix} \boldsymbol{y} \\ \boldsymbol{y}_{*} \end{bmatrix} \sim \mathcal{N} \left( \boldsymbol{0},  \begin{bmatrix} \boldsymbol{K}(\boldsymbol{X}, \boldsymbol{X}) + \sigma_{e}^{2} \boldsymbol{I}_{n}, & \!\!\!\!\!\!\!\! \boldsymbol{K}(\boldsymbol{X}, \boldsymbol{X}_{*}) \\ \boldsymbol{K}(\boldsymbol{X}_{*}, \boldsymbol{X}), & \!\!\!\!\!\!\!\! \boldsymbol{K}(\boldsymbol{X}_{*}, \boldsymbol{X}_{*}) + \sigma_{e}^{2} \boldsymbol{I}_{n_{*}} \end{bmatrix} \right), \nonumber 
\end{equation}
where $\boldsymbol{K}(\boldsymbol{X}, \boldsymbol{X})$ is an $n \times n$ matrix of covariances among the training inputs; $\boldsymbol{K}(\boldsymbol{X}, \boldsymbol{X}_{*})$ is an $n \times n_{*}$ matrix of covariances between the training inputs and test inputs; $\boldsymbol{K}(\boldsymbol{X}_{*}, \boldsymbol{X}_{*})$ is an $n_{*} \times n_{*}$ matrix of covariances among the test inputs.
%
%
Here, we let $\boldsymbol{K}(\boldsymbol{X}, \boldsymbol{X})$ be a short term of $\boldsymbol{K}(\boldsymbol{X}, \boldsymbol{X}; \boldsymbol{\theta}_h)$ when the kernel hyper-parameters have been trained and the associated optimization process is not the spotlight. 

Applying the results of conditional Gaussian distribution, we can easily derive the posterior distribution as
\begin{equation}
p(\boldsymbol{y}_{*} \vert \mathcal{D}, \boldsymbol{X}_{*}; \boldsymbol{\theta}_h) \sim
\mathcal{N} \left(  \bar{\boldsymbol{m}} , \bar{\boldsymbol{V}}  \right), 
\end{equation} 
where the posterior mean and posterior variance are respectively,
\begin{align}
\bar{\boldsymbol{m}} & = \boldsymbol{K}(\boldsymbol{X}_{*}, \boldsymbol{X}) \left[ \boldsymbol{K}(\boldsymbol{X}, \boldsymbol{X}) + \sigma_{e}^{2} \boldsymbol{I}_{n} \right]^{-1} \boldsymbol{y}, \\
\bar{\boldsymbol{V}} & = \boldsymbol{K}(\boldsymbol{X}_{*}, \boldsymbol{X}_{*}) + \sigma_{e}^{2} \boldsymbol{I}_{n_{*}} \nonumber \\ 
&- \boldsymbol{K}(\boldsymbol{X}_{*}, \boldsymbol{X}) \left[ \boldsymbol{K}(\boldsymbol{X}, \boldsymbol{X}) + \sigma_{e}^{2} \boldsymbol{I}_{n} \right]^{-1}  \boldsymbol{K}(\boldsymbol{X}, \boldsymbol{X}_{*}).
\end{align}
%

In general, temporal Gaussian processes take training input $\boldsymbol{x}_{t} = [x_{1,t}, x_{2,t},...,x_{d,t}]^T$ with discrete time index $t=1,2,...,n$, where $x_{1,t}, x_{2,t},...,x_{d,t}$ are specifically the $d$ features observed at time $t$. In this paper, we focus on the one-dimensional (1-D) time series with $d=1$ and $\boldsymbol{x}_{t} = x_{t} = t$. 

\subsection{Classic GP Hyper-parameter Optimization}
\label{subsec:HyperparaOpt}
Next, we introduce the classic ML based GP hyper-parameter estimation. Due to the Gaussian assumption on the noise, the log-likelihood function can be obtained in closed form. 
%
 %
%
The GP hyper-parameters can be tuned equivalently by minimizing the negative log-likelihood function (ignoring the unrelated terms) as 
\begin{equation}
\boldsymbol{\theta}_{ML} \triangleq \arg \min_{\boldsymbol{\theta}} \,  l(\boldsymbol{\theta}) = \boldsymbol{y}^T \boldsymbol{C}^{-1}(\boldsymbol{\theta}) \boldsymbol{y} +  \log \det \left( \boldsymbol{C}(\boldsymbol{\theta}) \right), 
\label{eq:likelihood}
\end{equation}
where $\boldsymbol{C}(\boldsymbol{\theta}) \triangleq \boldsymbol{K}(\boldsymbol{X}, \boldsymbol{X}; \boldsymbol{\theta}_h) + \sigma_{e}^{2} \boldsymbol{I}_{n}$. This optimization problem is mostly solved via gradient based methods, such as LFGS-Newton or conjugate gradient \cite{RW06}, which requires the following partial derivatives for $i=1,2,...,p$ in closed form:
\begin{equation}
\frac{\partial l(\boldsymbol{\theta}) }{\partial \theta_i} \!=\!  tr \!\left( \boldsymbol{C}^{-1}(\boldsymbol{\theta})  \frac{\partial \boldsymbol{C}(\boldsymbol{\theta})}{\partial \theta_i} \right)  -  \boldsymbol{y}^{T} \boldsymbol{C}^{-1}(\boldsymbol{\theta})  \frac{\partial \boldsymbol{C}(\boldsymbol{\theta})}{\partial \theta_i} \boldsymbol{C}^{-1}(\boldsymbol{\theta}) \boldsymbol{y}. \nonumber 
\end{equation}

\subsection{Linear Multiple Kernel}
\label{subsec:MLK}
Linear multiple kernel, as its name suggests, constitutes a linear combination of primitive kernels whose weights are to be optimized. In this paper, we solely focus on the scenario, in which the underlying kernel function $k(t, t')$ is \textbf{completely unknown} but is approximated as $k(t, t') \approx \sum_{i=1}^{m} \alpha_i k_{i}(t, t')$, where the basis sub-kernel functions $k_{i}(t, t')$, $i=1,2,...,m$ are known and the weights $\alpha_i$, $i=1,2,...,m$ are the optimization variables, subject to $\alpha_i \geq 0$. Often, the number of the basis sub-kernels, $m$, is set large to allow for good approximation. For this scenario, no expert knowledge is required. The associated kernel hyper-parameters are $\boldsymbol{\theta}_h = \boldsymbol{\alpha} = [\alpha_1, \alpha_2,...,\alpha_m]^T$. We will introduce two ways of constructing a grid spectral mixture kernel in Section~\ref{sec:GSMKernel} with the aim to let the data decide on the most favorable stationary kernel function approximated by a linear multiple of basis kernels.  

%

\section{Stationary Kernel Design in the Frequency Domain}
\label{sec:GSMKernel}
In subsection~\ref{subsec:SMKernel}, we first briefly review the spectral mixture (SM) kernel proposed originally in \cite{WA13} for approximating any stationary kernel while stressing out the associated difficulties when optimizing with respect to the hyper-parameters. In subsection~\ref{subsec:ProposedGSMKernel}, we introduce two ways of constructing grid spectral mixture (GSM) kernel for building 1-D temporal Gaussian process regression models. Lastly, we show how to combine Welch periodogram with $L_1$ norm regularization for advanced setup of the GSM kernel in subsection~\ref{subsec:GSM-advanced-setup}.
\subsection{SM Kernel \cite{WA13}}
\label{subsec:SMKernel}
The SM kernel undertakes approximation in the frequency domain using the fact that a stationary kernel function and its spectral density are Fourier duals due to the following corollary of Bochner's theorem given in \cite{RW06}. 
%
%
\begin{corollary}
For time series where the free variable is time, i.e., $\boldsymbol{x} = t$, $\tau = t-t'$, $f$ being the normalized frequency (i.e., $f \in [0, 1/2)$) and in the case that the spectral density $S(f)$ exists, the stationary kernel function, $k(\tau)$, and its spectral density of the kernel function, $S(f)$, are Fourier duals of each other as shown below:
\begin{subequations}
\begin{align}
k(\tau) &= \int_{\mathbb{R}^{1}} S(f) \exp \left[ j 2 \pi \tau f  \right] d f ,  \\ 
S(f) &= \int_{\mathbb{R}^{1}} k(\tau) \exp \left[- j 2 \pi \tau f  \right] d \tau .
\end{align}
\end{subequations}
\end{corollary}

The salient SM kernel is designed by approximating the spectral density, $S(f)$, of the underlying stationary kernel by a Gaussian mixture. 
%
%
%
%
%
Taking the inverse Fourier transform of $S(f)$, yields a stationary kernel in the time-domain as
\begin{equation}
k(t, t^{'}; \boldsymbol{\theta}_h) \!=\! k(\tau) \!=\!\! \sum_{q=1}^{Q} \! \alpha_q \exp \! \left[ -2 \pi^{2} \tau^2 \sigma_{q}^2 \right] \cos(2 \pi \tau \mu_q),
\label{eq:SM-Kernel}
\end{equation}
where $\boldsymbol{\theta}_h \triangleq [\alpha_1,...,\alpha_Q, \mu_1,...,\mu_Q, \sigma_{1}^2,..., \sigma_{Q}^2]^T$ denotes the SM kernel hyper-parameters with $Q$ being a fixed number of mixture components, and $\alpha_q$, $\mu_q$, $\sigma_{q}^{2}$ being the weight, mean and variance of the $q$-th mixture component, respectively. The SM kernel is able to approximate any stationary kernel arbitrarily well in $L_1$ norm according to the Wiener's theorem of approximation \cite{Achieser92}.

However, minimizing the negative log-likelihood with respect to $\boldsymbol{\theta}$ in light of Eq.(\ref{eq:likelihood}), it may easily get stuck at a bad local optimum, because the cost function is non-convex in terms of $\boldsymbol{\theta}$ and may not have any favorable structure to facilitate the optimization process.  
%
%

\subsection{Proposed GSM Kernel}
\label{subsec:ProposedGSMKernel}
To address the potential numerical problems with the original SM kernel, we proposed a GSM kernel in \cite{Yin18} with the goal to modify the original SM kernel by fixing the $\mu$ and $\sigma$ parameters to $\textit{a priori}$ selected values in a grid. To be precise, the spectral density is approximated by the GSM kernel as
\begin{equation}
S(f) = \sum_{i=1}^{m} \alpha_i s_{i}(f),
\label{eq:GSM-freqDomain}
\end{equation}
where each $s_{i}(f) = \mathcal{N}(f; \mu_i, \sigma_{i}^2) + \mathcal{N}(f; -\mu_i, \sigma_{i}^2)$ is evaluated at a fixed point in a grid $(\mu_i, \sigma_{i}^{2})$, sampled either uniformly or randomly from a two dimensional space confined in $[\mu_{low}, \mu_{high}]$ and $[\sigma^{2}_{low}, \sigma^{2}_{high}]$. The sampling strategies are shown in Fig.~\ref{fig:figure1} for clarity. 
Taking the inverse Fourier transform of the above spectral density, $S(f)$, yields our first GSM kernel formulation in the time domain as 
\begin{align}
k(t, t^{'}; \boldsymbol{\alpha}) &= \sum_{i=1}^{m} \alpha_{i} k_{i}(t, t^{'}) \nonumber \\
&= \! \sum_{i=1}^{m} \alpha_{i} \exp \! \left[ -2 \pi^{2} \tau^2 \sigma_{i}^2 \right] \! \cos(2 \pi \tau \mu_{i}), 
\label{eq:grid-SM-kernel}
\end{align}
where $\boldsymbol{\alpha}$ is a vector of the kernel hyper-parameters, namely the unknown non-negative weights. Since the grids are generated in the 2-D $(\mu, \sigma)$ space, the resultant kernel is called 2-D GSM kernel. 

The problem with the so designed 2-D GSM kernel lies in the large number of unknown parameters to be optimized. We note that the weights $\alpha_1, \alpha_2,...,\alpha_m$ are all non-negative numbers, but we will not constrain the sum $\sum_{i=1}^{m} \alpha_i$ to be equal to one for approximating a spectral density whose integral is not equal to one, as in \cite{Wilson14}. Therefore, we slightly abuse the term ``Gaussian mixture" mostly used for probability density approximation \cite{Bishop06}. 
\begin{figure}[t]
	\centering
	\includegraphics[width=.45\textwidth]{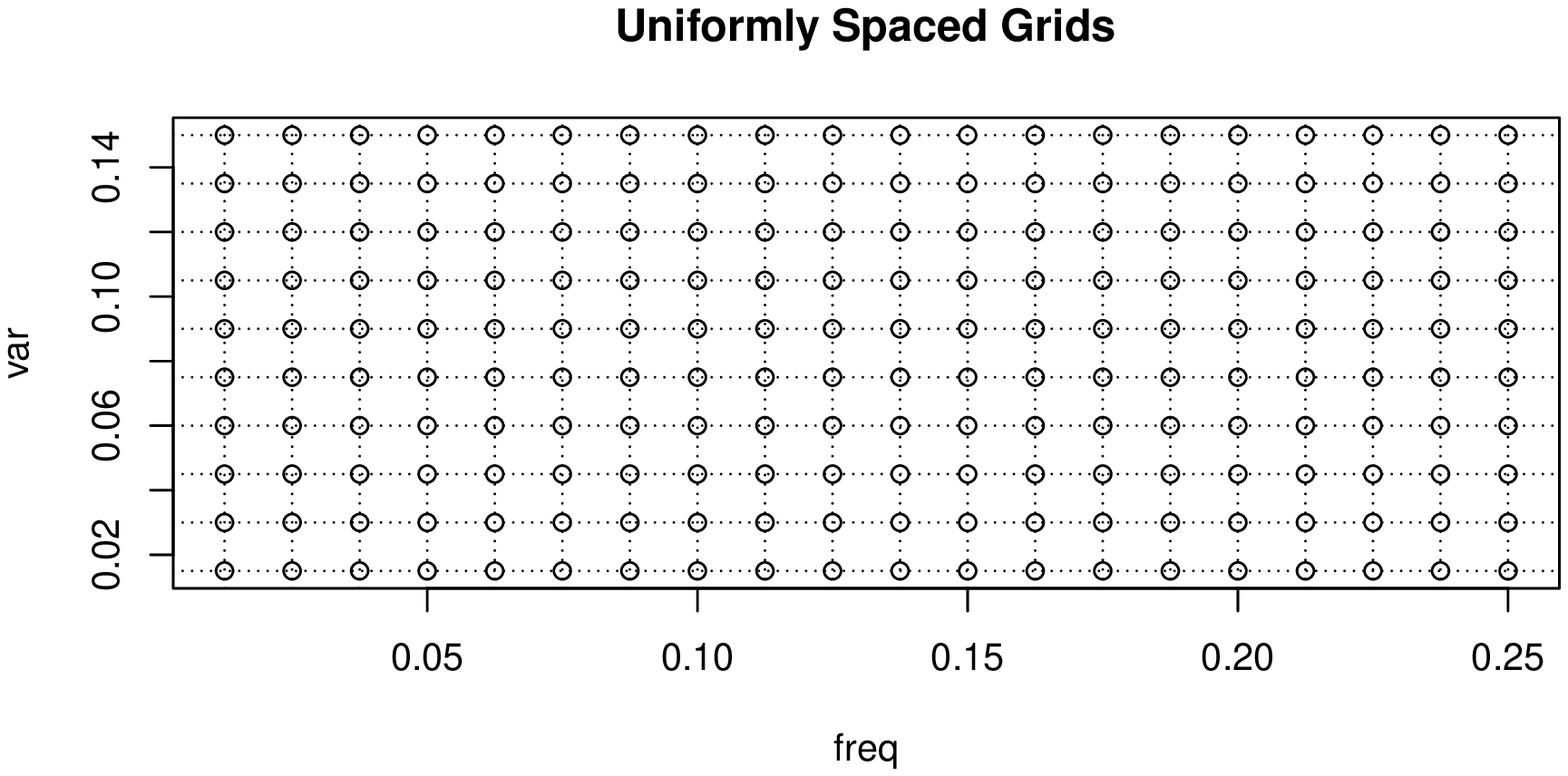}
	\includegraphics[width=.45\textwidth]{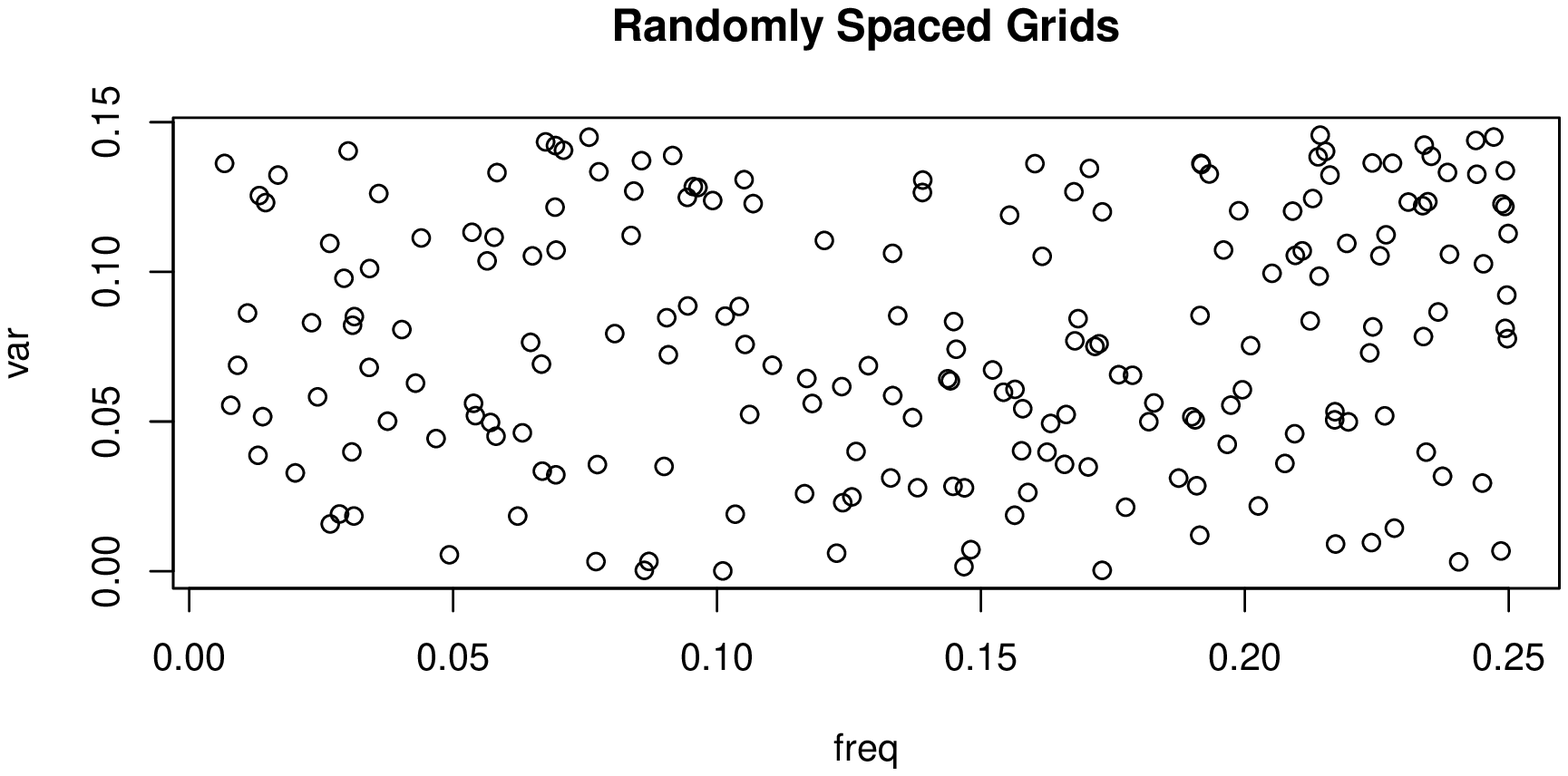}
	\caption{Illustration of the two strategies for generating grids. In this specific example, $\mu_{low}$ is set to be 0,  $\mu_{high} = 0.25$, $\sigma_{low} = 0$ and $\sigma_{high} = 0.15$.}
	\label{fig:figure1}
\end{figure}

In the following, we derive a modified 1-D GSM kernel by fixing the variance parameters, $\sigma_i$, in Eq.(\ref{eq:grid-SM-kernel}) to a small fixed value, $\sigma$, so as to reduce the high model complexity. The resultant GSM kernel boils down to
\begin{align}
k(t, t^{'}; \boldsymbol{\alpha}) = \! \sum_{i=1}^{m} \! \alpha_{i} \exp ( -2 \pi^{2} \tau^2 \sigma^2 ) \! \cos(2 \pi \tau \mu_{i}).
\label{eq:grid-SM-kernel-fixed-sigma}
\end{align}
To differentiate with the 2-D GSM kernel, the kernel given in Eq.(\ref{eq:grid-SM-kernel-fixed-sigma}) is called 1-D GSM kernel, because the grids are generated in the 1-D $\mu$-space, given a fixed $\sigma$. With this 1-D GSM kernel, the underlying spectral density, $S(f)$, is approximated by a linear weighted sum of Gaussian basis functions with varying shifts, $\mu_i$, while fixed bandwidth, $\sigma$. In addition to the kernel design, we also demonstrate some useful properties of the proposed GSM kernels. 
%
%


\begin{theorem}
Some properties of the GSM kernel in Eq.~(\ref{eq:grid-SM-kernel-fixed-sigma}) are given as follows:
\begin{enumerate}
\item It is a valid kernel.
\item It is smooth with derivatives of all orders.
\item Each one of the sub-kernel functions, $k_{i}(\tau)$, is square integrable for any $i=1,2,...,m$. 
\item For big data set with size $n \gg \frac{4}{\pi \sigma}$, the sub-kernel matrix is sparse and close to a band matrix with equal lower and upper bandwidths (irrespective of $\mu_i$), which enables more efficient utilization of computer memory, e.g., in MATLAB \cite{Gilbert92}.
\item For a given data set with $n$ samples, when the variance parameter, $\sigma$, is chosen sufficiently small, then for any frequency parameter $\mu_{i} \in [0, 1/2)$, the corresponding sub-kernel matrix has low rank, $\textrm{rank}(\boldsymbol{K}_i) \ll n$. 
\end{enumerate}
%
\end{theorem}

\begin{proof}
Sketch of the proofs are summarized below:
\begin{itemize}
\item Proof of property (1) is given in the Appendix A. 
\item Verification of property (2) is straightforward.
\item Reasoning of properties (3) and (4) is given in the supplement.
\item Proof of property (5) is given in the Appendix B. 
\end{itemize}
\end{proof}
It is easy to see that the above properties hold for any sub-kernel of the 2-D GSM kernel in Eq.(\ref{eq:grid-SM-kernel}) as well. 

\begin{remark}
Our way of constructing the 1-D GSM kernel is related, in some sense, to the non-parametric kernel density estimator using an optimal kernel width \cite{Silverman86}. The difference, however, lies in the distribution of the ``frequency variables". For the 1-D GSM kernel, the frequency variables $\mu_i$ are selected either uniformly or randomly from the selected region; while in the non-parametric kernel density estimation, the ``frequency variables" are essentially generated from the underlying density function to be reconstructed. 
\end{remark}

\subsection{Advanced Setup of the 1-D GSM Kernel}
\label{subsec:GSM-advanced-setup}
In the previous subsection, we have seen a modified GSM kernel function as given in Eq.~(\ref{eq:grid-SM-kernel-fixed-sigma}). In order to make it attractive from a practical point of view, as it will be confirmed in Section~\ref{sec:results}, one simply needs to 1) choose a moderate number of modes, $m$; 2) set a small number $\sigma$ common to all grids, and 3) sample $\mu_i$, $i=1,2,...,m$, either uniformly or randomly from $[0, 1/2)$. Naturally, one may ask for more advanced setup of the GSM kernel with reduced model complexity, promising sampling areas, and better initial guess of the unknown weights. 

For the purpose of obtaining an advanced setup, we could exploit the observations $y(t), t=1,2,...,n$, which are assumed to comprise a noisy realization of the underlying stationary random process $f(t)$, so that to build an estimate of the true spectral density. A candidate is to use the Welch periodogram \cite{Oppenheim1975} as an estimator of the underlying spectral density, $S(f)$. To construct a Welch periodogram, we need to partition $y(t), t=1,2,...,n$, into $L$ overlapped segments, $y_{l}(t), l=1,2,...,L$, each with only $D$ data points. For each segment, a local periodogram is then computed as 
\begin{equation}
I_{W,l}^{D}(f) = \frac{1}{DA} \left|  \sum_{t=1}^{D} w(t) y_{l}(t) e^{-j 2 \pi f t}  \right|^2, 
\end{equation}
where $w(t)$ is a deterministic window function, e.g., the Bartlett window, and $A = \frac{1}{D} \sum_{t=1}^{D} |w(t)|^2$ is a normalization factor. The final Welch periodogram, $\hat{S}_{W}(f)$, is given as the average of the $L$ local periodograms.
%
%
For a big data set with $n$, $L$, and $D$ all being large, the Welch periodogram is an asymptotically consistent estimator of the underlying power spectral density. Hence, by inspecting the Welch periodogram, we may obtain good prior knowledge about the model complexity, $m$, as well as the salient areas for the sampling points in the grid. 
%
%

Next, the previously obtained periodogram will be used to compute a potentially good initial guess of the weights, $\boldsymbol{\alpha}$. 
To this end, we solve
\begin{equation}
\min_{\boldsymbol{\alpha}} || \boldsymbol{s}_W - \boldsymbol{\Psi} \boldsymbol{\alpha} ||_{2}^{2} + \lambda || \boldsymbol{\alpha} ||_{1},
\label{eq:L1-norm-LS}
\end{equation}
where $ \boldsymbol{s}_W = [\hat{S}_{W}(\mu_1), \hat{S}_{W}(\mu_2),...,\hat{S}_{W}(\mu_m)]^T $ contains the periodogram values evaluated at the discrete frequencies, $\mu_1$, $\mu_2$,...,$\mu_m$; matrix $\boldsymbol{\Psi}$ is of size $m \times m$, whose $i$-th row is the transpose of $\boldsymbol{s}(\mu_i) = [s_{1}(\mu_i), s_{2}(\mu_i),...,s_{m}(\mu_i)]^T$ with each entry computed according to the definition of the Gaussian mixture component introduced in Eq.~(\ref{eq:GSM-freqDomain}). A practical and efficient method for solving a large scale $L1$-regularized least-squares problem in Eq.(\ref{eq:L1-norm-LS}) can be found in \cite{Kim07}. 

Here, we must acknowledge that using empirical periodogram for additional information concerning the underlying spectral density was already mentioned in \cite{WA13}. However, in our current context, it is used as a potentially better initial guess of $\boldsymbol{\alpha}$, given the knowledge that our hyper-parameter estimate will be sparse. 

\section{Memory Efficient Kernel Matrix Approximations}
\label{sec:GSM-Approx}
When the proposed GSM kernel contains a large number of sub-kernels, i.e., $m$ is large, and moreover the data size $n$ is large, unaffordable memory is needed to store the $m$ huge sub-kernel matrices during the hyper-parameter optimization process, as will be introduced in Section~\ref{sec:hyperpara-opt-GSMKernel}. Often, a factor $\boldsymbol{L}_i$, satisfying $\boldsymbol{K}_i = \boldsymbol{L}_i \boldsymbol{L}_i^{T}$, is stored instead of the sub-kernel matrix $\boldsymbol{K}_i$ with much reduced memory, especially when $\boldsymbol{K}_i $ has low rank. In this section, we discuss two kernel matrix approximations, namely the Nystr\"{o}m approximation in subsection~\ref{subsec:Nystrom} and the random Fourier feature approximation in subsection~\ref{subsec:RFF}, that can be adopted to provide good approximations of $\boldsymbol{L}_i$ with relatively low computational complexity and reduced memory. 
\subsection{Nystr\"{o}m Approximation \cite{WS01}}
\label{subsec:Nystrom}
First, we introduce the Nystr\"{o}m approximation \cite{WS01} of the kernel matrix $\boldsymbol{K}_{i}$, $i=1,2,...,m$. In the sequel, we omit the subscript $i$ for brevity because the same procedure can be applied to any $\boldsymbol{K}_{i}$. Detailed steps are as follows:

\textbf{Step 1}: Sample a subset of $p$ ($\leq n$) training inputs to form $\tilde{\boldsymbol{X}}$ from the complete set of training inputs $\boldsymbol{X}$. 

\textbf{Step 2}: Compute $\boldsymbol{K}^{(p)}$ with the sub-sampled training inputs $\tilde{\boldsymbol{X}}$. Herein, the superscript $(p)$ indicates $\boldsymbol{K}^{(p)}$ is of size $p \times p$. 

\textbf{Step 3}: Perform eigendecomposition of the smaller kernel matrix $\boldsymbol{K}^{(p)}$ as
\begin{equation}
\boldsymbol{K}^{(p)} = \sum_{l=1}^{\tilde{p}} \lambda_{l}^{(p)} \boldsymbol{u}_{l}^{(p)} \left( \boldsymbol{u}_{l}^{(p)} \right)^{T},
\end{equation}
where $\tilde{p}$ denotes the effective number of eigenvalues that are distinctly larger than zero and obviously $\tilde{p} \leq p$. We further define $\boldsymbol{\Sigma}^{(p)} \triangleq diag(\lambda_{1}^{(p)}, \lambda_{2}^{(p)},...,\lambda_{\tilde{p}}^{(p)})$ and $\boldsymbol{U}^{(p)} \triangleq \left[ \boldsymbol{u}_{1}^{(p)}, \boldsymbol{u}_{2}^{(p)},...,\boldsymbol{u}_{\tilde{p}}^{(p)}\right]$ for later use. 

\textbf{Step 4}: Apply Nystr\"{o}m approximation to the eigenvalues and eigenvectors obtained in the previous step as follows:
\begin{align}
\tilde{\lambda}_{l} &= \frac{n}{p} \lambda_{l}^{(p)}, \quad l=1,2,...,\tilde{p},  \\
\tilde{\boldsymbol{u}}_{l} &= \sqrt{\frac{p}{n}} \frac{1}{\lambda_{l}^{(p)}} \boldsymbol{K}(\boldsymbol{X}, \tilde{\boldsymbol{X}}) \boldsymbol{u}_{l}^{(p)}, \quad l=1,2,...,\tilde{p},
\end{align}
where $\tilde{\lambda}_{l}$ and $\tilde{\boldsymbol{u}}_{l}$ are respectively the approximated $l$-th eigenvalue and eigenvector of the original $n \times n$ kernel matrix $\boldsymbol{K}$, and $\boldsymbol{K}(\boldsymbol{X}, \tilde{\boldsymbol{X}})$ is an $n \times p$ matrix of correlations between the training inputs $\boldsymbol{X}$ and sub-sampled training inputs $\tilde{\boldsymbol{X}}$. 

\textbf{Step 5}: Finally, we obtain a low-rank (of rank $\tilde{p}$) approximation of the original kernel matrix $\boldsymbol{K}$ as follows:
\begin{equation}
\boldsymbol{K} \approx \tilde{\boldsymbol{K}} = \sum_{l=1}^{\tilde{p}} \tilde{\lambda}_{l} \tilde{\boldsymbol{u}}_{l} \tilde{\boldsymbol{u}}_{l}^{T} = \tilde{\boldsymbol{U}} \tilde{\boldsymbol{\Sigma}} \tilde{\boldsymbol{U}}^{T},
\label{eq:Ny-approx-Kernel-Matrix}
\end{equation}
where $\tilde{\boldsymbol{U}} \triangleq \left[ \tilde{\boldsymbol{u}}_{1}, \tilde{\boldsymbol{u}}_{2},...,\tilde{\boldsymbol{u}}_{\tilde{p}} \right]$ is the matrix of the $\tilde{p}$ eigenvectors and $\tilde{\boldsymbol{\Sigma}}\triangleq diag(\tilde{\lambda}_{1}, \tilde{\lambda}_{2},...,\tilde{\lambda}_{\tilde{p}}) $ is a diagonal matrix of the $\tilde{p}$ eigenvalues. Lastly, we approximate the factor $\boldsymbol{L}$ by $\tilde{\boldsymbol{L}} \triangleq \tilde{\boldsymbol{U}} \tilde{\boldsymbol{\Sigma}}^{1/2}$, which is of smaller size $n \times \tilde{p}$.

It is easy to verify that the memory usage for storing $\tilde{\boldsymbol{L}}$ is reduced to $\tilde{p}/n \times 100\%$ of the original usage for storing $\boldsymbol{L}$. Moreover, the computational complexity for performing eigendecomposition is also reduced from $\mathcal{O}(n^3)$ to $\mathcal{O}(p^3)$. 

%

\subsection{Random Fourier Feature Approximation \cite{Rahimi07}}
\label{subsec:RFF}
Next, we introduce the random Fourier feature approximation. When the spectral density function $S(f)$ is an even function of $f$, we can easily derive the corresponding stationary kernel function as
\begin{equation}
k(t, t') = \mathbb{E}_{S(f)} \left[ \cos \left( 2 \pi f t - 2 \pi f t'  \right) \right] . 
\end{equation}
By replacing the above integration with Monte-Carlo integration, $k(t, t')$ is approximated as follows:
\begin{align}
k(t, t') &\approx \frac{1}{R} \sum_{r=1}^{R} \cos \left( 2 \pi f_r t - 2 \pi f_r t'  \right),
\end{align}
where $f_r$, $r=1,2,...,R$ are sampled from the spectral density function $S(f)$. Let $\omega_r \triangleq 2 \pi f_r$, $r=1,2,...,R$ and define
\begin{equation}
\boldsymbol{\phi}_{\omega}(t) \!\! \triangleq \!\! \frac{1}{\sqrt{R}} [\cos(\omega_1 t), \sin(\omega_1 t),...,\cos(\omega_{R} t), \sin(\omega_{R} t)]^T
\end{equation}
we have $k(t, t') \approx \boldsymbol{\phi}^{T}_{\omega}(t)  \boldsymbol{\phi}_{\omega}(t')$.  
%

The GSM kernel proposed in the above subsection is of form $k(t, t^{'}; \boldsymbol{\alpha}) = \sum_{i=1}^{m} \alpha_{i} k_{i}(t, t^{'})$
and it can be approximated as 
\begin{equation}
k(t, t' ; \boldsymbol{\alpha}) \approx \sum_{i=1}^{m} \alpha_i \boldsymbol{\phi}^{T}_{i, \omega}(t)  \boldsymbol{\phi}_{i, \omega}(t'),
\end{equation} 
by applying random Fourier representation to each sub-kernel function, i.e.,
\begin{equation}
k_{i}(t, t') \!=\!  \exp \! \left[\! -2 \pi^{2} \tau^2 \sigma_{i}^2 \right] \! \cos(2 \pi \tau \mu_{i}) \! \approx \! \boldsymbol{\phi}^{T}_{i, \omega}(t)  \boldsymbol{\phi}_{i, \omega}(t'),
\end{equation}
where the random Fourier features for the $i$-th sub-kernel $k_{i}(t, t')$ are randomly sampled from $s_{i}(f)$ (also a valid distribution function). 

With the aid of the random Fourier feature representation of the sub-kernel, the overall GSM kernel matrix can be represented as $\boldsymbol{K} = \sum_{i=1}^{m} \alpha_i \boldsymbol{K}_i \approx \tilde{\boldsymbol{K}} = \sum_{i=1}^{m} \alpha_i \tilde{\boldsymbol{L}}_i \tilde{\boldsymbol{L}}_{i}^{T}$, where $\tilde{\boldsymbol{L}}_i = [\boldsymbol{\phi}_{i,\omega}(t_1), \boldsymbol{\phi}_{i,\omega}(t_2),....,\boldsymbol{\phi}_{i,\omega}(t_n)]^{T}$ is an $n \times 2R$ matrix and $\boldsymbol{L}_i \approx \tilde{\boldsymbol{L}}_i$. The memory usage for storing $\tilde{\boldsymbol{L}}$ is reduced to $2R/n \times 100\%$ of the original usage for storing $\boldsymbol{L}$. The computational complexity is mainly due to a batch of samplings from a Gaussian distribution, which remains low. Random Fourier feature is widely used for resource limited kernel approximations, e.g., for fast online learning \cite{Bouboulis18} and others \cite{Theodoridis15}. Alternative to the random Fourier features, one may also use the Fastfood features that can be computed more efficiently \cite{Le13}. 

\subsection{RAE versus Storage}
\label{subsec:error-vs-storage}
To compare the above two methods in terms of approximation accuracy and storage, we adopt the widely used metric relative approximation error (RAE) as the performance metric, which is given by $\Vert{\boldsymbol{K}-\tilde{\boldsymbol{K}}}\Vert_F / \Vert{\boldsymbol{K}}\Vert_F$, where $\boldsymbol{K}$ is the exact kernel matrix and $\tilde{\boldsymbol{K}} = \tilde{\boldsymbol{L}} \tilde{\boldsymbol{L}}^T$ is its approximation. Due to space limitation, the results are shown in the supplement. 

The general conclusions are as follows:
\begin{itemize}
\item For small data set, Nystr\"{o}m approximation may require less memory than the random Fourier feature approximation in order to achieve a similar small value of RAE, say less than $1\%$. 
\item For medium or large data set, random Fourier feature approximation may require less memory than the Nystr\"{o}m approximation in order to achieve a similar small value of RAE. This is because the number of random features needed for constructing a good approximation can be kept to several hundreds, not sensitive to sample size of the selected data set; while the number of data points needed by the Nystr\"{o}m approximation increases with the sample size, in general. When the kernel matrices have low rank, less samples is needed by the Nystr\"{o}m approximation. 
\end{itemize}

\section{ML Based GP Hyper-parameter Optimization}
\label{sec:hyperpara-opt-GSMKernel}
For linear multiple kernel, including the proposed GSM kernel, the associated maximum-likelihood based GP hyper-parameter optimization problem can be cast as
\begin{align}
\boldsymbol{\theta}_{ML} \!=\! \arg \min_{\boldsymbol{\alpha}, \sigma_{e}^2} \left\lbrace \boldsymbol{y}^T \boldsymbol{C}^{-1}(\boldsymbol{\alpha}, \sigma_{e}^2) \boldsymbol{y} \!+\!  \log \det \! \left( \boldsymbol{C}(\boldsymbol{\alpha}, \sigma_{e}^2) \right) \right\rbrace,
\label{eq:MLnew}
\end{align}
subject to $\boldsymbol{\alpha} \geq \boldsymbol{0}$ and $\sigma_{e}^2 \geq 0$. Here, $\boldsymbol{\theta}= [\boldsymbol{\alpha}^T, \sigma_{e}^2]^{T}$ and $\boldsymbol{C}(\boldsymbol{\alpha}, \sigma_{e}^2) \triangleq  \sum_{i=1}^{m} \alpha_{i} \boldsymbol{K}_{i} +  \sigma_{e}^2 \boldsymbol{I}_{n}$, where $\boldsymbol{K}_{i}$ is the $i$-th sub-kernel matrix evaluated at the grid point $(\mu_i, \sigma_i)$ for the 2-D GSM kernel or $(\mu_i, \sigma)$ for the 1-D GSM kernel.
%
%
%
The cost function in Eq.(\ref{eq:MLnew}) is a difference of two convex functions with respect to $\boldsymbol{\alpha}$ and $\sigma_{e}^2$, therefore the optimization problem belongs to the well known difference-of-convex program (DCP) \cite{BV04,CAL14,Lipp16}. Here, we want to stress, once more, that the primary idea behind the newly proposed GSM kernel is to maintain good approximation capability with a structure that leads to a well known optimization problem with respect to the GP hyper-parameters. However, ML method for the previously suggested SM kernel leads to a general non-convex hyper-parameter optimization task. The additional structure in Eq.(\ref{eq:MLnew}) can facilitate the optimization task in terms of convergence speed and avoidance of bad local minimum, as will be seen in our experiments. 
%

In the following, we derive two numerical methods for optimizing the GP hyper-parameters. In subsection~\ref{subsec:DCP}, we derive a sequential majorization-minimization (MM) method. In subsection~\ref{subsec:ADMM-solution}, we derive a nonlinearly constrained alternating direction method of multipliers (ADMM) \cite{Boyd11}. No matter which method is adopted, the solution can be proven to be sparse according to the theorem provided along with some other properties in subsection~\ref{subsec:analysisMIN}. 
\subsection{Sequential MM Method \cite{CAL14}}
\label{subsec:DCP}
%
%
%
The main idea of the MM method is to solve $\min_{\boldsymbol{\theta} \in \Theta}\, l(\boldsymbol{\theta})$ with $\Theta \subseteq \mathbb{R}^{m+1}$ through an iterative scheme, where at each iteration a so-called majorization function $\bar{l}(\boldsymbol{\theta},\boldsymbol{\theta}^{k})$ of $l(\boldsymbol{\theta})$ at $\boldsymbol{\theta}^{k} \in \Theta$ is minimized, i.e.,
\begin{equation}
\boldsymbol{\theta}^{k+1} = \arg \min_{\boldsymbol{\theta} \in \Theta} \bar{l}(\boldsymbol{\theta},\boldsymbol{\theta}^{k}),
\label{eq:MM}
\end{equation}
where $\bar{l}: \Theta \times \Theta \to \mathbb{R}$ satisfies $\bar{l}(\boldsymbol{x},\boldsymbol{x}) = l(\boldsymbol{x})$ for $\boldsymbol{x} \in \Theta$ and $l(\boldsymbol{x}) \le \bar{l}(\boldsymbol{x},\boldsymbol{z})$ for $\boldsymbol{x}, \boldsymbol{z} \in \Theta$. For this particular DCP problem in Eq.(\ref{eq:MLnew}), $l(\boldsymbol{\theta}) = g(\boldsymbol{\theta}) - h(\boldsymbol{\theta})$, where $g(\boldsymbol{\theta}) = \boldsymbol{y}^T \boldsymbol{C}(\boldsymbol{\theta})^{-1} \boldsymbol{y}$, $h(\boldsymbol{\theta}) = -\log \det \boldsymbol{C}(\boldsymbol{\theta})$ and $\boldsymbol{C}(\boldsymbol{\theta}) = \sum_{i=1}^{m} \alpha_{i} \boldsymbol{L}_{i} \boldsymbol{L}_{i}^{T} +  \sigma_{e}^2 \boldsymbol{I}_{n}$; $g(\boldsymbol{\theta}), h(\boldsymbol{\theta}): \Theta \to \mathbb{R}$ are convex and differentiable functions with $\Theta$ being a convex set in $\mathbb{R}^{m+1}$. 
Here, we use the so-called linear majorization, i.e., we make the convex function $h(\boldsymbol{\theta})$ affine by performing the first-order Taylor expansion and obtain $\bar{l}(\boldsymbol{\theta},\boldsymbol{\theta}^{k}) = g(\boldsymbol{\theta}) - h(\boldsymbol{\theta}^{k}) - \nabla_{\boldsymbol{\theta}}^{T} h(\boldsymbol{\theta}^{k}) (\boldsymbol{\theta} - \boldsymbol{\theta}^{k})$.
Hence, at each iteration, minimizing the cost function in Eq.(\ref{eq:MM}) becomes a convex optimization problem. The MM method is guaranteed to converge to a stationary point when some regularization conditions are satisfied \cite{CAL14}. But it is noticed that solving this problem directly with \textit{CVX}, a package for specifying and solving convex programs \cite{GB13}, is very computationally demanding. 

Since $g(\boldsymbol{\theta})$ is a matrix fractional function, each iteration in the MM method actually solves a convex matrix fractional minimization problem. 
Due to the fact that $ \sum_{i=1}^{m} \alpha_{i} \boldsymbol{L}_{i} \boldsymbol{L}_{i}^{T} + \sigma_{e}^2 \boldsymbol{I}_{n} $ is a sum of positive semi-definite terms, the SDP problem can be further cast as a conic quadratic optimization problem with $m+1$ rotated quadratic cone constraints, i.e.,
\begin{align}
& \min_{\boldsymbol{z}, \boldsymbol{\theta}, \boldsymbol{v}, \boldsymbol{w}} \quad 2(\boldsymbol{1}^{T} \boldsymbol{z}) - \nabla_{\boldsymbol{\theta}}^{T} h(\boldsymbol{\theta}^{k}) \boldsymbol{\theta} \nonumber \\
s.t. & \quad \Vert \boldsymbol{w}_{i} \Vert^{2}_{2} \le 2 \theta_{i} z_{i}, \quad i = 1,2,\ldots,m \nonumber \\
& \quad \Vert \boldsymbol{v} \Vert^{2}_{2} \le 2 z^\prime \nonumber \\
& \quad \boldsymbol{y} = \sum_{i=1}^{m} \boldsymbol{L}_{i} \boldsymbol{w}_{i} + \sigma_{e} \boldsymbol{v}, \quad \boldsymbol{\theta} \ge \boldsymbol{0}, \boldsymbol{z} \ge \boldsymbol{0}
\label{eq:second-order-cone}
\end{align}
where $\boldsymbol{\theta} \in \mathbb{R}^{m+1}, \boldsymbol{z} = [z_{1}, z_{2}, \cdots, z_{m}, z^{\prime}]^{T} \in \mathbb{R}^{m+1}, \boldsymbol{v} \in \mathbb{R}^{n}$, and $\boldsymbol{w}_{i} \in \mathbb{R}^{n_{i}}$ for $i = 1,2,\ldots,m$.
The conic quadratic optimization problem here is equivalent to a second-order cone program that can be solved efficiently using the commercial solver MOSEK \cite{CAL14}. 
\begin{algorithm}[t]
	\caption{Sequential MM Method}
	\KwIn{$\boldsymbol{y}$ and $\boldsymbol{L}_{i}, i = 1,2,\ldots,m$}
	\KwOut{$\boldsymbol{\theta}_{ML}$}
	Initialization: $k=0$, $\boldsymbol{\theta}^{0}$\\
	\While{the convergence condition is not satisfied}{
		Compute $h(\boldsymbol{\theta}^{k})$ and the gradient $\nabla_{\boldsymbol{\theta}} h(\boldsymbol{\theta}^{k})$.
		
		Solve Eq.(\ref{eq:second-order-cone}) for $\boldsymbol{\theta}^{k+1}$.
		
		Set $k = k+1$.
	}
	$\boldsymbol{\theta}_{ML} = \boldsymbol{\theta}^{k}$
\end{algorithm}

\begin{remark}
The computational complexity for solving one iteration of the above second-order cone program scales as $\mathcal{O}(n^2 \cdot max(n, \sum_{i=1}^{m} n_i))$, where $n_i$ stands for the rank of $\boldsymbol{K}_i$. The worst case complexity is $\mathcal{O}(m n^3)$ if all GSM sub-kernel matrices have full rank. Fortunately, as it was reported in \cite{CAL14} the MM method requires only a few iterations to achieve a good local optimum in practice.
\end{remark}
\begin{remark}
The above MM method matches perfectly with the proposed GSM kernel. This is due to the fourth property of the GSM kernel given in Theorem 1, i.e., for a given number of data samples, $n$, and a sufficiently small $\sigma$, the rank of $\boldsymbol{L}_i$ satisfies $n_i \ll n$, making $\sum_{i=1}^{m} n_i$ relatively small. Moreover, the two matrix approximation approaches are also helpful for reducing the complexity. For instance, using the random Fourier feature approximation can reduce the computational complexity to $\mathcal{O}(n^2 \cdot max(n, 2mR))$, where $R$ is the number of random features, specified in Section~\ref{sec:GSM-Approx}. 
\end{remark}

\subsection{Nonlinearly Constrained ADMM} 
\label{subsec:ADMM-solution}
In this subsection, we will propose a nonlinearly constrained ADMM for solving the optimal GP hyper-parameters, $\boldsymbol{\theta}$, from the maximum-likelihood estimation problem in Eq.~(\ref{eq:MLnew}). This new method has good potential to find a better local minimum with smaller negative likelihood value, $l(\boldsymbol{\theta})$, and the prediction MSE as compared to the sequential MM method and the classic gradient descent method. However, this method constrains itself to time series with short data records because its sub-problems involve matrix inversion and matrix multiplications, which scale as $\mathcal{O}(n^3)$ in general.

The idea is as follows. We reformulate the original problem by introducing an $n \times n$ matrix $\boldsymbol{S}$ and solve instead 
\begin{equation}
\arg \min_{\boldsymbol{S}, \boldsymbol{\alpha}} \boldsymbol{y}^{T} \boldsymbol{S} \boldsymbol{y} - \log \det(\boldsymbol{S}),
\end{equation}
subject to $\boldsymbol{S} \left( \sum_{i}^{m} \alpha_{i} \boldsymbol{K}_{i} + \sigma_{e}^2 \boldsymbol{I}_{n} \right) = \boldsymbol{I}_{n}$ and $\boldsymbol{\alpha} \geq \boldsymbol{0}$. Although $\sigma_{e}^2 \geq 0$ can be estimated jointly, we simply assume it is known \textit{a priori} and focus on the kernel hyper-parameters, $\boldsymbol{\alpha}$. This is for ease of notation and narration in the sequel.

The augmented Lagrangian function is then formulated as:
\begin{align}
L_{\rho}\left(\boldsymbol{S},  \boldsymbol{\alpha}, \boldsymbol{\Lambda} \right) &= \boldsymbol{y}^{T} \boldsymbol{S} \boldsymbol{y} - \log \det(\boldsymbol{S}) \nonumber \\ 
&+  \left\langle  \boldsymbol{\Lambda}, \boldsymbol{S} \left( \sum_{i}^{m} \alpha_{i} \boldsymbol{K}_{i} + \sigma_{e}^2 \boldsymbol{I}_{n} \right) - \boldsymbol{I}_{n} \right \rangle  \nonumber \\
&+ \frac{\rho}{2} \left| \left| \boldsymbol{S} \left( \sum_{i}^{m} \alpha_{i} \boldsymbol{K}_{i} + \sigma_{e}^2 \boldsymbol{I}_{n} \right) - \boldsymbol{I}_{n} \right| \right|_{F}^{2},
\label{eq:ADMM-ALF}
\end{align}
where the regularization parameter $\rho>0$ is fixed \textit{a priori}. The ADMM applied to Eq.(\ref{eq:ADMM-ALF}) iteratively decomposes into solving the following sub-problems: 
\begin{align}
\boldsymbol{S}^{k+1} &= \arg \min_{\boldsymbol{S}} L_{\rho} \! \left(\boldsymbol{S},  \boldsymbol{\alpha}^{k}, \boldsymbol{\Lambda}^{k} \right) \label{eq:ADMM-subproblem1} \\
\alpha_{i}^{k+1} &= \arg \min_{\alpha_i}  L_{\rho} \!\left(\boldsymbol{S}^{k+1}, \{ \alpha_i, \boldsymbol{\alpha}^{k,k+1}_{-i} \}, \boldsymbol{\Lambda}^{k} \right), i=1,...,m  \label{eq:ADMM-subproblem2} \\
\boldsymbol{\Lambda}^{k+1} &= \boldsymbol{\Lambda}^{k} \!+\! \rho^{\prime} \! \left[ \! \boldsymbol{S}^{k+1} \! \left( \sum_{i}^{m} \alpha_{i}^{k+1} \boldsymbol{K}_{i} + \sigma_{e}^2 \boldsymbol{I}_{n} \right) \!-\! \boldsymbol{I}_{n} \! \right],  \label{eq:ADMM-subproblem3}
\end{align}
where $\boldsymbol{\alpha}^{k,k+1}_{-i} \triangleq [\alpha_1^{k+1}, \alpha_2^{k+1},...,\alpha_{i-1}^{k+1}, \alpha_{i+1}^{k},...,\alpha_{m}^{k}]^{T}$ in Eq.(\ref{eq:ADMM-subproblem2}).
%
%

\begin{remark}
It is not difficult to verify that the subproblems of the proposed ADMM in Eq.~(\ref{eq:ADMM-subproblem1}) and Eq.~(\ref{eq:ADMM-subproblem2}) are both convex in terms of the corresponding optimization variables. 
\end{remark}

\begin{remark}
Different from the conventional ADMM, in the shown nonlinearly constrained ADMM, $\rho^{\prime}$ used for the dual variable update in Eq.(\ref{eq:ADMM-subproblem3}) is chosen to be smaller than $\rho$ used for the primal update in Eq.(\ref{eq:ADMM-ALF}). This novel configuration was first applied in the flexible proximal ADMM for consensus problems in \cite{Hong16}, where the authors set $\rho = \rho^{\prime} + L$ with $L$ being the Lipschitz constant of the the gradient of the objective function and harvested improved convergence performance. 
\end{remark}
In order to avoid the high computational cost for solving $\boldsymbol{S}^{k+1}$ precisely from Eq.(\ref{eq:ADMM-subproblem1}), which involves solving a quadratic matrix equation, we resort to the steepest descent method, which is computationally cheaper. We numerically update 
\begin{align}
\boldsymbol{S}^{k+1,\eta+1} = \boldsymbol{S}^{k+1,\eta} + \mu^{\eta} d^{\eta}, \quad \eta = 0,1,...,It_{\boldsymbol{S}}-1, 
\end{align}
where $d^{\eta} = -\frac{\nabla_{\boldsymbol{S}}L_{\rho}}{\left| \left| \nabla_{\boldsymbol{S}}L_{\rho} \right| \right|_{F}}$, $S^{k+1,0} := S^{k}$, and $It_{\boldsymbol{S}}$ is a fixed number of inner iterations. The gradient of $L_{\rho}\left(\boldsymbol{S}, \boldsymbol{\alpha}^{k}, \boldsymbol{\Lambda}^{k} \right)$ with respect to $\boldsymbol{S}$, short as $\nabla_{\boldsymbol{S}}L_{\rho}$, is equal to 
\begin{align}
\nabla_{\boldsymbol{S}} L_{\rho}\left(\boldsymbol{S}, \boldsymbol{\alpha}^{k}, \boldsymbol{\Lambda}^{k} \right) &=  2 \boldsymbol{y} \boldsymbol{y}^{T} \!-\! (\boldsymbol{y} \boldsymbol{y}^{T}) \circ \boldsymbol{I}_n \!-\! 2 \boldsymbol{S}^{-1} \!+\! \boldsymbol{S}^{-1} \circ \boldsymbol{I}_n \nonumber \\ 
&+ \! \boldsymbol{\Lambda}^{k} \boldsymbol{C}^{k} + (\boldsymbol{\Lambda}^{k} \boldsymbol{C}^{k})^T \!-\! (\boldsymbol{\Lambda}^{k} \boldsymbol{C}^{k}) \circ \boldsymbol{I}_n \nonumber \\
&+ \! \rho ( \boldsymbol{S} \boldsymbol{C}^{k} \boldsymbol{C}^{k} \!+\! \boldsymbol{C}^{k} \boldsymbol{C}^{k} \boldsymbol{S}  \!-\! (\boldsymbol{S} \boldsymbol{C}^{k} \boldsymbol{C}^{k}) \circ \boldsymbol{I}_n)\nonumber \\
&- \! \rho (2\boldsymbol{C}^{k} - \boldsymbol{C}^{k} \circ \boldsymbol{I}_n),
\label{eq:Original_S_gradient}
\end{align}
where both $\boldsymbol{C}^{k} = \sum_{i}^{m} \alpha_{i}^{k} \boldsymbol{K}_{i} + \sigma_{e}^2 \boldsymbol{I}_{n} $ and $\boldsymbol{S}$ are symmetric. The above gradient involves a matrix inverse of current $\boldsymbol{S}^{k}$ which is computationally demanding. For speed up, we replace $(\boldsymbol{S}^{k})^{-1}$ with $\boldsymbol{C}^{k}$ as approximation in Eq.(\ref{eq:Original_S_gradient}) whenever possible.
%
%
%
%
To be precise, at each iteration, we stick to the approximated gradient if $|| \boldsymbol{S}^{k} \boldsymbol{C}^{k} - \boldsymbol{I} ||_{F} \leq \delta$, where $\delta$ is a manually selected threshold to trade-off approximation error and computational time; Otherwise, the original gradient in Eq.~(\ref{eq:Original_S_gradient}) will be used. The stepsize $\mu^{\eta}$ is selected according to Armijo rule \cite{Bertsekas16} at each iteration. More details about the stepsize selection can be found in the supplement. 
%


For solving $\alpha_i$ from Eq.(\ref{eq:ADMM-subproblem2}), we take the derivative of $L_{\rho}\left(\boldsymbol{S}^{k+1}, \{\alpha_i, \boldsymbol{\alpha}^{k,k+1}_{-i} \}, \boldsymbol{\Lambda}^{k} \right)$ with respect to $\alpha_i$, $\forall i=1,2,...,m$ and set it equal to zero, yielding
\begin{align}
& \left\langle \boldsymbol{\Lambda}^{k}, \boldsymbol{S}^{k+1} \boldsymbol{K}_{i} \right\rangle + \rho \left[ \alpha_{i} \cdot tr\left( \boldsymbol{K}_{i}^{T} \boldsymbol{S}^{k+1, T} \boldsymbol{S}^{k+1} \boldsymbol{K}_{i} \right) \right]  \nonumber \\
& + \rho \cdot tr \left[  \left( \tilde{\boldsymbol{K}}_{-i} \boldsymbol{K}_i + \sigma_{e}^2 \boldsymbol{K}_{i} \right) \boldsymbol{S}^{k+1, T} \boldsymbol{S}^{k+1} - \boldsymbol{S}^{k+1} \boldsymbol{K}_{i} \right] = 0.  \label{eq:alpha_derivative} 
\end{align}
where $\tilde{\boldsymbol{K}}_{-i} = \sum_{j=1}^{i-1} \alpha_{j}^{k+1} \boldsymbol{K}_{j} + \sum_{j=i+1}^{m} \alpha_{j}^{k} \boldsymbol{K}_{j}$.
%
%
Following the steps sketched in Appendix C, $\alpha_{i}^{k+1}$ can be re-expressed as
\begin{equation}
\alpha_{i}^{k+1} \!\!=\!\! \left[ \! \alpha_{i}^{k} \!+\! \frac{tr \! \left[ \boldsymbol{K}_{i} \boldsymbol{S}^{k+1} \! \left( \! \boldsymbol{I}_n \!-\! \boldsymbol{S}^{k+1}  \tilde{\boldsymbol{C}}_{i}^{k+1} \!-\! \frac{1}{\rho} \boldsymbol{\Lambda}^{k} \!\right)  \right] }{tr\left( \boldsymbol{K}_{i} \boldsymbol{S}^{k+1} \boldsymbol{S}^{k+1} \boldsymbol{K}_{i} \right)} \! \right]_{+}
\label{eq:alpha_update}
\end{equation}
where
\begin{equation}
\tilde{\boldsymbol{C}}_{i}^{k+1} = \sum_{j=1}^{i-1} \alpha_{j}^{k+1} \boldsymbol{K}_{j} + \sum_{j=i}^{m} \alpha_{j}^{k} \boldsymbol{K}_{j} +\sigma_{e}^{2} \boldsymbol{I}_{n}.
\label{eq:C_tilde}
\end{equation}
It is noted that $\boldsymbol{K}_{i}^{T} = \boldsymbol{K}_{i}$ and $\left( \boldsymbol{S}^{k+1} \right)^T = \boldsymbol{S}^{k+1}$ due to the symmetric property of a kernel matrix. For clarity, we provide detailed steps for implementing the proposed nonlinearly constrained ADMM in Algorithm 2. 


\begin{algorithm}[t]
	\caption{Proposed Nonlinearly Constrained ADMM}
	\KwIn{$\boldsymbol{y}$, $\sigma_{e}^{2}$, and $\boldsymbol{K}_{i}, i = 1,2,\ldots,m$}
	\KwOut{$\boldsymbol{\alpha}_{ML}$}
	Initialization: $k=0$, $\boldsymbol{\alpha}^{0}$, $\boldsymbol{\Lambda}^{0}$, $\rho$, $\rho'$, $\epsilon_{ADMM}$, $\epsilon_{\boldsymbol{S}}$, $It_{\boldsymbol{S}}$. \\ 
	Set $\boldsymbol{C}^{0} =  \sum_{i}^{m} \alpha_{i}^{0} \boldsymbol{K}_{i} + \sigma_{e}^2 \boldsymbol{I}_{n} $, $\boldsymbol{S}^{0} = \left[\boldsymbol{C}^{0}\right]^{-1}$\\
	\For{(outer iterations) $k=0,1,...$}{
	    $\eta = 0$, $ \boldsymbol{S}^{k+1, \eta=0} = \boldsymbol{S}^{k}$ \\
		\For{(inner iterations) $\eta=0,1,...,It_{\boldsymbol{S}}-1$}{
			
			1. Compute $d^{\eta} = -\frac{\nabla_{\boldsymbol{S}}L_{\rho}}{\left| \left|  \nabla_{\boldsymbol{S}}L_{\rho}  \right| \right|_{F}}$ analytically according to Eq.(\ref{eq:Original_S_gradient}) or its approximation obtained using $\boldsymbol{C}^{k}$ to replace the inverse of $\boldsymbol{S}^k$.
					    
			2. Adopt Armijo rule to select the step size $\mu^{\eta}$ and perform:
			  
			$\boldsymbol{S}^{k+1, \eta+1} = \boldsymbol{S}^{k+1, \eta} + \mu^{\eta} d^{\eta}$
			
			\If{$\left| \left| \boldsymbol{S}^{k+1, \eta+1} - \boldsymbol{S}^{k+1, \eta} \right| \right|_{F} \leq \epsilon_{\boldsymbol{S}}$}{
			
			$\eta = \eta+1$
			
			\textbf{break}
			}			
			
			$\eta = \eta+1$
			
		}
		
		Update $\boldsymbol{S}^{k+1} = \boldsymbol{S}^{k+1, \eta}$	
			
		\For{$i =1$ to $m$}{
		Compute $\alpha^{k+1}_{i}$ analytically according to Eq.~(\ref{eq:alpha_update}).

        Compute $\tilde{\boldsymbol{C}}^{k+1}_{i} $analytically according to Eq.~(\ref{eq:C_tilde}).
		}
		\If{$\left| \left| \boldsymbol{\alpha}^{k+1} - \boldsymbol{\alpha}^{k} \right| \right| \leq \epsilon_{ADMM}$}{
		
		$\boldsymbol{\alpha}_{ML} = \boldsymbol{\alpha}^{k+1}$
		
		\Return
		}
		Update $\boldsymbol{C}^{k+1} = \tilde{\boldsymbol{C}}^{k+1}_{m}$
		
		Update $\boldsymbol{\Lambda}^{k+1}$ analytically according to Eq.~(\ref{eq:ADMM-subproblem3}).
		
		Set $k = k+1$.
	}
	
	$\boldsymbol{\alpha}_{ML} = \boldsymbol{\alpha}^{k}$
\end{algorithm}

%
\begin{remark}
When taking the initial guess $\boldsymbol{\Lambda}^{0}$ close to the optimal Lagrange multiplier $\boldsymbol{\Lambda}^{*}$ and taking $\rho$ sufficiently large, solving the unconstrained minimization problem $L_{\rho}(\boldsymbol{S}, \boldsymbol{\alpha}, \boldsymbol{\Lambda})$ can yield points close to the local minimum $\boldsymbol{S}^{*}$ and $\boldsymbol{\alpha}^{*}$ that satisfy the sufficient optimality conditions. Details can be found in sections 4.2 and 5.2 of \cite{Bertsekas16}.
\end{remark}

\subsection{Properties of the Optimized Hyper-Parameters}
\label{subsec:analysisMIN}
This subsection aims to give some additional properties of the optimized GP hyper-parameters from both the statistical signal processing and optimization perspectives. 

\begin{theorem}
Global minimum $(\boldsymbol{\alpha}^{*}, \sigma^{2}_{e,*})$ exists that leads Eq.(\ref{eq:MLnew}) to minus infinity, when the output satisfies $\boldsymbol{y} = \boldsymbol{V}_{i} \boldsymbol{z}$ for $\boldsymbol{z} \in \mathbb{R}^{p}$, $|| \boldsymbol{z} ||_{2}^{2} < \infty$, and $p<n$, where the $n \times p$ matrix $\boldsymbol{V}_{i} \triangleq \boldsymbol{U}_{i} \boldsymbol{\Sigma}_{i}^{1/2}$ with $\boldsymbol{\Sigma}_{i}^{1/2}$ being the diagonal matrix of square-root of the $p$ non-zero eigenvalues and $\boldsymbol{U}_{i}$ of size $n \times p$ containing the corresponding eigenvectors of a rank-deficient sub-kernel matrix $\boldsymbol{K}_i$. 
\end{theorem}
\begin{proof}
The proof can be found in \cite{Yin18} or in the supplement.
\end{proof}

\begin{theorem}
Every local minimum of Eq.(\ref{eq:MLnew}) is achieved at a sparse solution, regardless of whether noise is present or not. 
\end{theorem}
\begin{proof}
See \cite[Theorem 2]{WR04}. 
\end{proof}

\begin{remark}
The sparseness of the ML solution according to the above theorem is celebrating for the proposed 1-D GSM kernel. The reasons are twofold. First, it means that only the frequencies thought to be important for modeling by the data will be pinpointed, endowing good interpretation of the kernel. Second, by avoiding to use all the grids (or model freedom) to fit the data, over-parameterization problem as indicated by Proposition 1 can be effectively alleviated. 
\end{remark}

\section{Experimental Results}
\label{sec:results}
In this section, we aim to investigate the prediction performance of the proposed GSM kernel based GP and compare it with the SM kernel based GP proposed by Wilson \textit{et.al.} in \cite{WA13} and the sparse spectrum GP proposed by L\'{a}zaro-Gredilla \textit{et.al.} in \cite{Gredilla10} from various aspects. We picked up in total 8 classic time series data sets for test. Descriptions of the data are shown in Table~\ref{tab:table1}. The training data, $\mathcal{D}$, is used for optimizing the GP hyper-parameters; while the test data, $\mathcal{D}_{*}$, is used for evaluating the prediction MSE. 
%
\begin{table}[h]
\begin{center}
   \caption{Details of the selected data sets.}
   \label{tab:table1}
    \begin{tabular}{|c|p{0.34\columnwidth}|c|c|}
    \hline
    Name & Description & Training $\mathcal{D}$ & Test $\mathcal{D}_{*}$ \\ \hline
    ECG & Electrocardiography of an ordinary person measured over a period of time & 680 & 20 \\ \hline
    CO2 & CO2 concentration made between 1958 and the end of 2003 & 481 & 20 \\ \hline
    Electricity & Monthly average residential electricity usage in Iowa City 1971-1979 & 86 & 20 \\ \hline
    Employment & Wisconsin employment time series, trade, Jan. 1961 – Oct. 1975 & 158 & 20 \\ \hline
    Hotel & Monthly hotel occupied room average 1963-1976  & 148 & 20 \\ \hline
    Passenger & Passenger miles (Mil) flown domestic U.K., Jul. 1962-May 1972 & 98 & 20 \\ \hline 
    Clay & Monthly production of clay bricks: million units. Jan 1956 – Aug 1995 & 450 & 20 \\ \hline 
    Unemployment & Monthly U.S. female (16-19 years) unemployment figures (thousands) 1948-1981 & 380 & 20 \\ \hline 
    \end{tabular}
\end{center}
\end{table}

\subsection{Algorithmic Setup}
\noindent For the proposed GSM kernel based GP, short for GSMGP, we provide its setup in each individual subsection. Source code and all test data are available online.\footnote{\url{https://github.com/Paalis/MATLAB_GSM}}
%
%
%

\noindent The SM kernel based GP, short for SMGP, proposed by Wilson \textit{et.al.}:
\begin{itemize}
\item We use the source code provided on the author's web page and follow the default setup suggested therein.\footnote{\url{https://people.orie.cornell.edu/andrew/code/}}
\item We follow the initialization strategy given on the author's web page as well. Random restart is, however, not used.
\item The number of Gaussian mixture components $Q$ is chosen to be 10 or 500 for the SM kernel. 
\item The SMGP model hyper-parameters are determined by a gradient-descent type method. 
\end{itemize}

\noindent The Sparse spectrum (SS) GP, short for SSGP, proposed by L\'{a}zaro-Gredilla \textit{et.al.}:
\begin{itemize}
\item We use the source code provided on the author's web page and follow the default setup \footnote{\url{http://www.tsc.uc3m.es/~miguel/downloads.php}}.
\item We follow the strategy given by the authors to initialize the hyper-parameters. The number of basis is set to $m = 500$ in the simulations.
\item The SSGP model hyper-parameters are determined by a conjugate-gradient method.
\end{itemize}
It is noteworthy that the independent noise variance parameter $\sigma_{e}^2$ is estimated using the cross-validation filter type method\cite{Garcia10} and it is kept common to all above GP models for fair comparisons. In the following experiments, we solely compare the performance of the GSM kernel, SM kernel, and SS kernel. In \cite{WA13, Wilson14, Gredilla10}, extensive experiments with both synthesized and real data have confirmed the effectiveness of the SM kernel and SS kernel as compared to the elementary kernels such as the SE kernel and Matern kernel. 

\subsection{Performance of the 2-D GSM kernel with MM Method}
This subsection is a wrap-up of the results obtained in \cite{Yin18} for the 2-D GSM kernel using a big number of grids generated from 2-D space. Therein, the test involved 30 independent Monte-Carlo (MC) runs, and in each MC run, a new set of 20,000 grid points were randomly generated in the 2-D $(\mu-\sigma)$ space confined by $\mu_{low} = 0$, $\mu_{high} = 0.5$, $\sigma^2_{low} = 0$ and $\sigma^{2}_{high} = 0.15$. We initialize the weights of the sub-kernels, $\boldsymbol{\alpha}$, to a vector of zeros for the 2-D GSM kernel. The prediction MSE is evaluated for all selected GP models. Besides, we count the number of MC runs (out of 30 in total), in which one method stucked at a bad/meaningless local minimum (i.e., does not provide a meaningful prediction) and calculate the ratio, referred to as program fail rate (PFR) in this paper. Note that, the meaningless results were excluded when we compute the MSE. 
\begin{table}[t]
\begin{center}
    \caption{Performance comparison between the proposed GSMGP (with 2-D grids) and its competitors, SSGP and SMGP, in terms of the MSE and the PFR.}
    \label{tab:table2}
    \begin{tabular}{|c|c|c|c|c|c|}
    \hline
    Name & SSGP & SMGP & SMGP & GSMGP & GSMGP \\ 
             & MSE & MSE & PFR & MSE & PFR \\ \hline
    ECG &1.6E-01 & 2.1E+00 & 0.63 & NA & NA \\ \hline
    CO2    &2.0E+02 & 7.4E+04 & 0.83 & NA & NA \\ \hline
    Electricity &8.2E+03 & 1.8E+04 & 0.47 & 6.8E+03 & 0.2\\ \hline
    Employment &7.7E+01 & 2.3E+04 & 0.27 & 3.9E+01 & 0.07\\ \hline
    Hotel  &1.9E+04 &2.6E+05 & 0.33 & 2.4E+03 & 0\\ \hline
    Passenger &6.9E+02 & 3.5E+03 & 0.37 & 1.7E+02 & 0\\ \hline
    Clay &5.3E+02 & 4.8E+03 & 0.93 & NA & NA\\ \hline
    Unemploy &2.1E+04 & 1.2E+05 & 0.9 & NA & NA\\ \hline
    \end{tabular}
\end{center}
\end{table}

From the results shown in Table~\ref{tab:table2}, we can conclude that the proposed 2-D GSM kernel based GP regression has gained well improved prediction MSE and stability as compared to its competitors. We did not show the PFR of the SSGP becasue it can always get the trend/envelop of the data but fail to fit small-scaled, fine structures. Whereas, the SMGP using $Q=10$ Gaussian modes can better fit the data with a good starting point but it may even fail to capture the trend of the data with a bad starting point. The performance of the proposed GP model becomes better and more stable, when the number of the grids grows beyond around 10,000. In Table~\ref{tab:table2}, the results of the GSMGP on the \textit{CO2, clay, and unemployment} data sets are not available since the large size of the unknown weights, $dim(\boldsymbol{\alpha})$, and long data record jointly make the program beyond the processing capability of our computer.\footnote{Specifications: Intel(R) Core(TM) i7-8700 CPU 3.2GHz, 3192MHz, 6 cores, 16GB RAM with MATLAB2017a installed} Apart from the improved performance, the average number of non-zero $\boldsymbol{\alpha}$ values generated by the ML method is equal to 26, 19, 17, 22, respectively for the four data sets that can be handled. These results confirm with Theorem 3, claiming that the ML solution of our estimation problem is sparse. 

The average computational time for the MM method to solve the GP hyper-parameters in one MC run is around 1 minute, 25 minutes, 10 minutes, 9 minutes, respectively for the four smaller data sets that can be handled. From next subsection on, we will solely focus on the new 1-D GSM kernel with much reduced model complexity.

\subsection{Performance of the 1-D GSM Kernel with MM Method}
In the previous subsection, we showed the performance of the GSM kernel with 2-D grids. The model complexity, $m$, is expected to be large for good performance. We need to reduce the model complexity. We resort to the GSM kernel with 1-D grids as given in Eq.(\ref{eq:grid-SM-kernel-fixed-sigma}), for which we sample $m$ frequency parameters $\mu_i$, $i=1,2,...,m$ uniformly from the given frequency region $[0,1/2)$, while fix the variance parameter to a small constant, $\sigma = 0.001$. The GP hyper-parameters are solved via the sequential MM method, for which the initial guess of $\alpha_i$, is first generated from a Gaussian distribution with zero mean and large variance, say $\sigma^2_{\alpha} = 10$, and then finalized by $\max(\alpha_i, 0)$. Random restart is not used for fair comparison. 
\begin{table}[t]
\begin{center}
    \caption{Prediction MSE generated by two GSM kernels (one is using $m=20000$ 2-D grids vs. the other using $m=500$ 1-D grids).}
    \label{tab:table3}
    \begin{tabular}{|c|c|c|c|c|c|}
    \hline
    Name &1-D & 1-D & 1-D & 2-D & 2-D\\ 
             & MSE & Iterations & PFR & MSE & Iterations \\ \hline
    ECG & 1.3E-02 & 24 &0.01 & NA & NA \\ \hline
    CO2 & 1.5E+00 & 10 &0.17 & NA & NA \\ \hline
    Electricity & 4.7E+03 & 2 &0.07 & 6.8E+03 & 2 \\ \hline
    Employment & 1.1E+02 & 23 &0.06 & 3.9E+01 & 14 \\ \hline
    Hotel & 8.9E+02 & 14 &0.02 & 2.4E+03 & 6 \\ \hline
    Passenger & 1.9E+02 & 28 &0.02 & 1.7E+02 & 13 \\ \hline
    Clay  & 1.9E+02 & 25 & 0.12 & NA & NA \\ \hline
    Unemploy. & 3.6E+03 & 9 &0.10 & NA & NA \\ \hline
    \end{tabular}
\end{center}
\end{table}
\begin{table}[h]
	\begin{center}
		\caption{Prediction MSE of the GSMGP with $m=500$ 1-D grids vs. SMGP with $Q=500$ Gaussian modes.}
		\label{tab:table7}
		\begin{tabular}{|c|c|c|c|c|c|}
			\hline
			Name &GSMGP & GSMGP & SMGP & SMGP & SMGP\\ 
			& MSE & CT (s) & MSE & CT (s) &PFR \\ \hline
			ECG & 1.3E-02 & 140.4 & 1.9E-02 & 3.4E+03 & 0.3 \\ \hline
			CO2 & 1.5E+00 & 69.3 & 1.1E+00 & 2.0E+03 & 0.07 \\ \hline
			Electricity & 4.7E+03 & 1.46 & 7.5E+03 & 1.0E+02 & 0 \\ \hline
			Employment & 1.1E+02 & 31.2 & 0.7E+02 & 2.5E+02 & 0.03\\ \hline
			Hotel & 8.9E+02 & 17.5 & 2.8E+03 & 2.8E+02 & 0.97 \\ \hline
			Passenger & 1.9E+02 & 14.7 & 1.6E+02 & 1.1E+02 & 0.23\\ \hline
			Clay  & 1.9E+02 & 140.4 & 3.3E+02 & 3.4E+03 & 0 \\ \hline
			Unemploy. & 3.6E+03 & 42.3 & 1.4E+04 & 1.4E+03 & 0.57\\ \hline
		\end{tabular}
	\end{center}
\end{table}

To shed some light on its performance, we let $m=100, 200, 300, 400, 500$ and repeat the tests as conducted in the previous subsection for each $m$. We compare the prediction MSE obtained by the 2-D GSM kernel with 20,000 grids randomly sampled from the 2-D $(\mu, \sigma)$-space and the 1-D GSM kernel with only 500 grids uniformly selected from the 1-D $\mu$-space (with a fixed $\sigma=0.001$). The results are shown in Table~\ref{tab:table3}. In total 100 independent MC runs were conducted to compute the program fail rate as well as the prediction MSE after excluding the meaningless estimates. To better visualize the results, we show the training and prediction performance of the resulting GSMGP on the \textit{Electricity} and \textit{Unemployment} data sets in one specific MC run in Fig.~\ref{fig:figure2}. Similar results for all data sets are given in the supplement. 
\begin{figure*}[t]
\centering
\includegraphics[width=.45\textwidth]{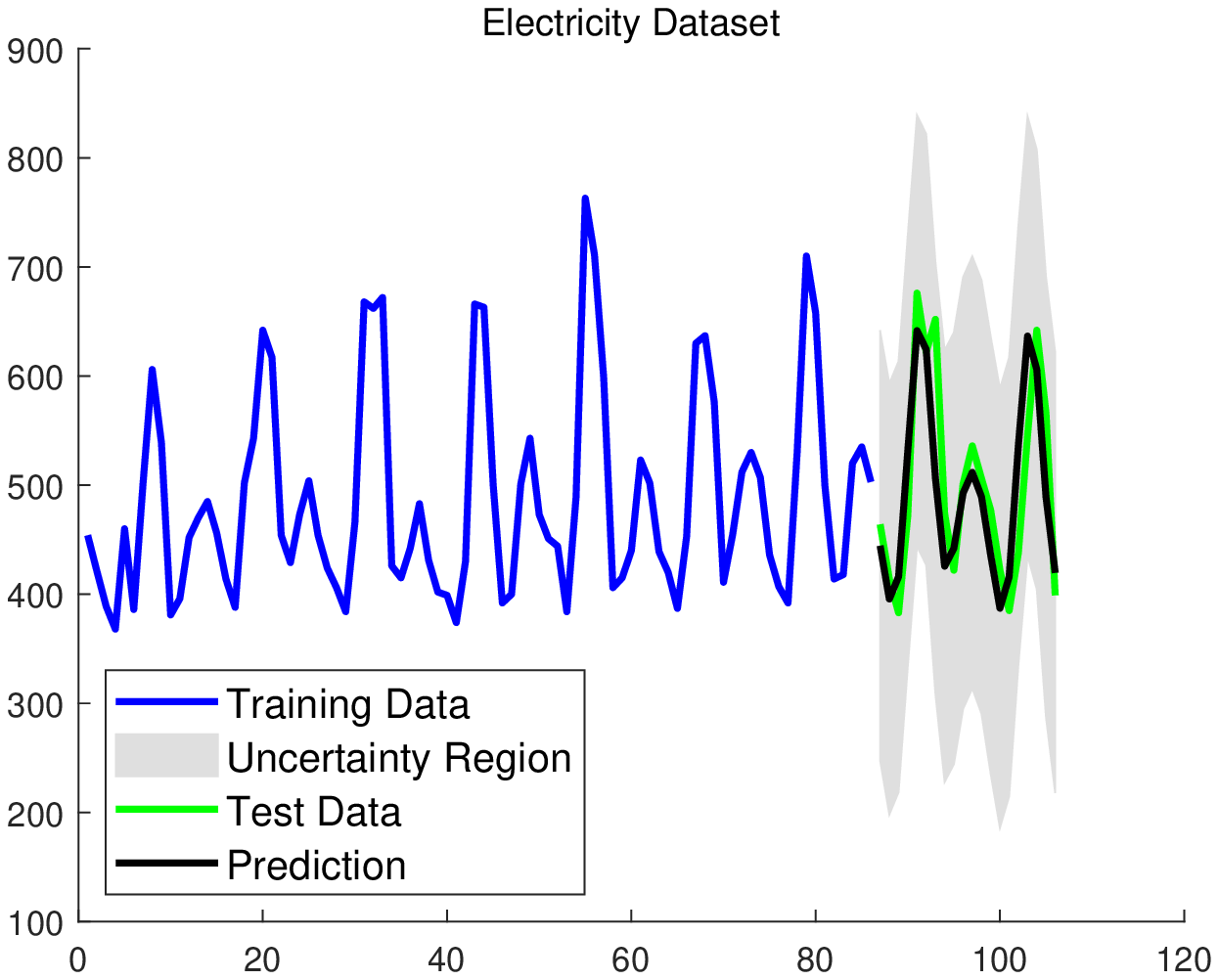}\quad 
\includegraphics[width=.45\textwidth]{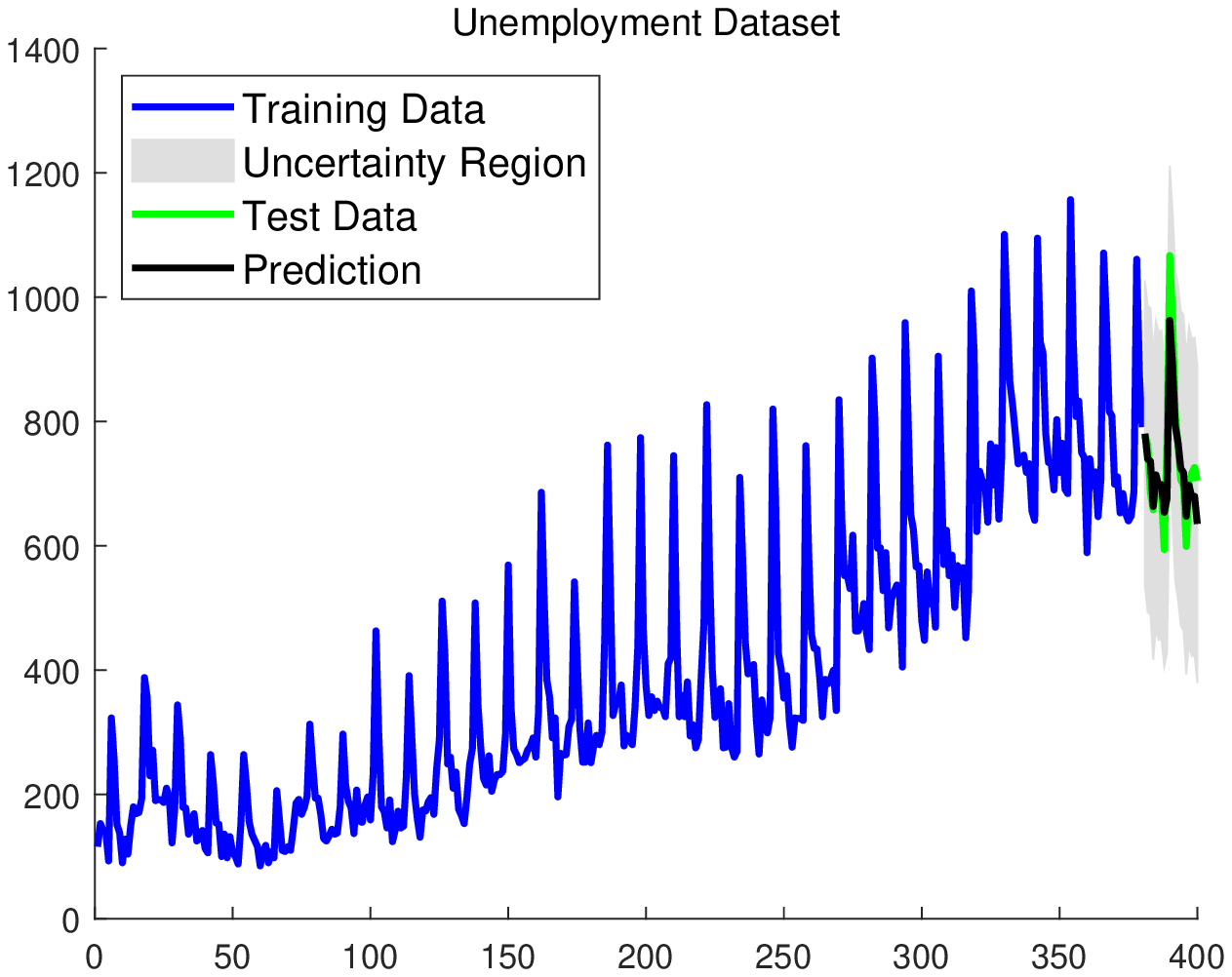}
\caption{Training and test performance of the GSMGP using 1-D GSM kernel with $\sigma = 0.001$ and $m=500$ uniformly generated grids. The optimal weights are solved via the MM method.}
\label{fig:figure2}
\end{figure*}
\begin{figure}[t]
	\includegraphics[width=0.45\textwidth]{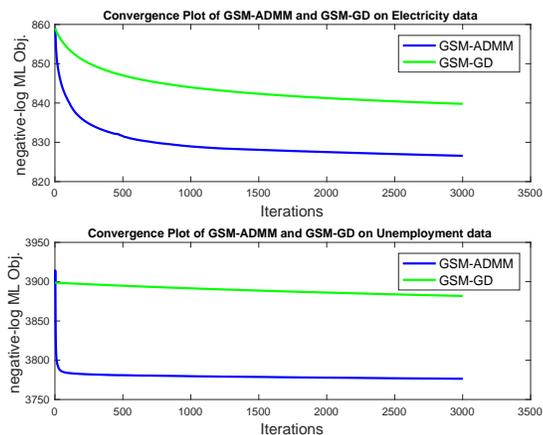}
	\caption{Negative log-likelihood versus iterations of the proposed ADMM as compared to the classic gradient projection.}
	 \label{fig:convergencePlot}
\end{figure}
\begin{figure}[t]
	\includegraphics[width=0.45\textwidth]{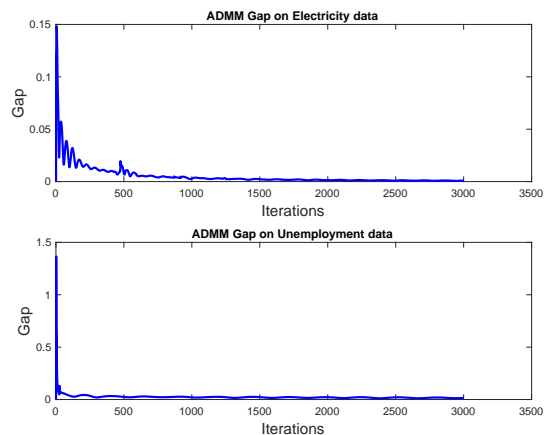}
	\caption{Gap of the equality constraint, $|| \boldsymbol{S}^{k} \boldsymbol{C}^{k} - \boldsymbol{I} ||_{F}$ versus the iterations of the proposed ADMM.}
	 \label{fig:gapEquality}
\end{figure}

Some observations from our experimental results are as follows. First, in a majority of cases, the prediction MSE generated by the 1-D GSM kernel degrades slightly as compared to that generated by the 2-D GSM kernel. This result is not surprising as the latter better covers the parameter space. On the other hand, the 2-D GSM kernel may overfit the training data in some cases, as seen for the \textit{Electricity} data set in Table~\ref{tab:table3}. Second, due to the significantly reduced model complexity, the MM method can handle much longer time series with the 1-D GSM kernel. Although the number of iterations required by the MM method increases for the 1-D GSM kernel, the overall computational time for the eight data sets has been reduced significantly from several minutes to several seconds. The surprisingly low computational time is also due to the low-rank property of all GSM sub-kernel matrices with $\sigma=0.001$, supported by Theorem 1. Lastly, although not shown in the Table, the performance of the 1-D GSM kernel becomes better and more stable as $m$ increases to around 500 grids but further increment would not help much for the selected data sets. Moreover, in Table~\ref{tab:table4}, we compare the GSMGP with an upgraded SMGP with $Q=500$, which leads to much improved prediction MSE and more stable numerical solution than that of $Q=10$ Gaussian modes, however at the cost of much longer computational time. But still for some data sets, e.g., the \textit{Unemployment} and \textit{Hotel}, the SMGP gets stuck at bad local minimal more frequently than our GSMGP. As a summary, the new 1-D GSM kernel has achieved overall better prediction results with much less reduced computational time and higher stability as compared to the original SM kernel with a large number of Gaussian modes.

In the following, we show the benefits of using Nystr\"{o}m to further speed up the computations. We still stick to the GSM kernel with $m=500$ 1-D grids and fixed $\sigma = 0.001$. We randomly sample only $5\%$ percent of the complete training inputs for constructing a Nystr\"{o}m approximation of every sub-kernel matrix $\boldsymbol{K}_{i}, i=1,2,...,m$. The results in Table~\ref{tab:table4} show the prediction MSE as well as the computational time that the MM method requires to converge in one particular MC run initialized with all zeros. The total computation time is not reduced much in this case because the sub-kernel matrices have low-rank (refer to the fifth property of the GSM kernel as given by Eq.(\ref{eq:grid-SM-kernel-fixed-sigma})) and this nice property matches perfectly with the MM method according to our remark 3 given in Section~\ref{sec:hyperpara-opt-GSMKernel}. For all data sets, we computed the rank of all sub-kernel matrices numerically and recorded the maximum rank, the minimum rank and the mean rank in Table~\ref{tab:table4b} which demonstrate $\max_{i} rank \{ \boldsymbol{K}_i \} \approx 2 \min_{i} rank \{ \boldsymbol{K}_i \} \approx 1.3 \sqrt{n}$; the mean rank is fairly close to the maximum rank because most of the sub-kernels have rank close to the maximum rank. Therefore, in light of the remark 3 of Section~\ref{sec:hyperpara-opt-GSMKernel}, the computational complexity of the MM method is approximately $\mathcal{O}(m n^{3/2})$ instead of the worst case $\mathcal{O}(mn^3)$. When we handle longer time series, random Fourier feature approximation may help save more memory while maintain similar RAE, as was shown in the supplement.
\begin{table}[t]
\begin{center}
    \caption{Prediction MSE generated by the 1-D GSM kernel versus its Nystr\"{o}m approximation, short as NY-GSM.}
    \label{tab:table4}
    \begin{tabular}{|c|c|c|c|c|c|}
    \hline
    Name &GSM & GSM  & NY-GSM  & NY-GSM \\ 
             & MSE & CT   & MSE & CT \\ \hline
    ECG & 1.3E-02 & 122s  & 1.3E-02 & 116s \\ \hline
    CO2 & 9.3E-01 & 24s  & 9.3E-01 & 22s\\ \hline
    Electricity & 3.0E+03 & 0.9s & 3.0E+03 & 0.2s\\ \hline
    Employment & 6.8E+01 & 12s & 6.8E+01 & 5s\\ \hline
    Hotel & 4.3E+02 & 3s & 4.3E+02 & 1s\\ \hline
    Passenger & 2.4E+02 & 8s & 2.9E+02 & 3s\\ \hline
    Clay  & 8.5E+01 & 60s & 8.5E+01 & 50s\\ \hline
    Unemploy. & 2.3E+03 & 8s & 2.3E+03 & 3s\\ \hline
    \end{tabular}
\end{center}
\end{table}
\begin{table}[h]
\begin{center}
    \caption{Maximum rank,  minimum rank, and mean rank of the selected $m=500$ GSM sub-kernel matrices used in the above experiments.}
    \label{tab:table4b}
    \begin{tabular}{|c|c|c|c|c|}
    \hline
    Name &$\max$ rank & $\min$ rank & mean rank  \\ 
             & GSM sub-kernels & sub-kernels  & sub-kernels\\ \hline
    ECG & 34 & 17 & 33 \\ \hline
    CO2 & 27 & 13 & 25  \\ \hline
    Electricity & 14 & 7 & 13 \\ \hline
    Employment & 16 & 8 & 15  \\ \hline
    Hotel & 14 & 7 & 13 \\ \hline
    Passenger & 14 & 7 & 13 \\ \hline
    Clay  & 26 & 13 & 25\\ \hline
    Unemployment & 24 & 12 & 23\\ \hline
    \end{tabular}
\end{center}
\end{table}

In all above experiments, we use the default setup of the 1-D GSM kernel, which is very simple to use. But as we pointed out in Section~\ref{sec:GSMKernel}, using the nonparametric Welch periodogram of the data to guide an advanced setup may be beneficial in various aspects. Due to space limitation, we show the periodogram of each data set versus the spectral density constructed using the optimal weights obtained for one specific MC in the supplement. As we can see, the periodogram indeed provides rich information for configuring the GSM kernel and optimizing its associated hyper-parameters.

\subsection{Performance of the 1-D GSM Kernel with ADMM}
In section~\ref{subsec:ADMM-solution}, we introduced a nonlinearly constrained ADMM, short as GSM-ADMM, for optimizing the hyper-parameters of the 1-D GSM kernel, i.e., the weights $\boldsymbol{\alpha}$. In the following experiments, we aim to compare it with other two numerical methods, namely the classic gradient projection (details see our supplement) and the sequential MM method, short as GSM-GD and GSM-MM, respectively. The performance is measured in terms of the objective function value and the prediction MSE.  

We conduct some experiments on a small data set and a moderate data set, as the proposed method is not suitable for big data set due to the $\mathcal{O}(n^3)$ complexity. To keep alignment with the previous experiments, we stick to the 1-D GSM kernel with $m = 500$ grids uniformly sampled from $\left[0,1/2\right)$ and fixed $\sigma = 0.001$. 

The algorithmic setup of our nonlinearly constrained ADMM as given in Algorithm~2 are in order. To update $\boldsymbol{S}$, we let $It_{\boldsymbol{S}} = 1000, \epsilon_{\boldsymbol{S}} = 10^{-15}, \delta = 1$. For selecting the step size in light of the Armijo rule, we let $s = 10^{-4}, \beta = 1/5, h = 10^{-5}$. The remainders are $\rho = 100, \rho^{\prime} = \rho/2 = 50, \epsilon_{ADMM} = 10^{-3}$. 

As for the initial guess, we let $\boldsymbol{\Lambda}^{(0)} = \boldsymbol{I}$, for the \textit{Electricity} data set; $\boldsymbol{\alpha}^{(0)}$ is obtained by fitting the nonparametric Welch periodogram via the $L_1$-norm regularized least-squares mentioned in Section~\ref{sec:GSMKernel}, while for the \textit{Unemployment} data set, $\boldsymbol{\alpha}^{(0)}$ is obtained by running just one iteration of the sequential MM method. The same initial guesses were applied to the GSM-GD for fair comparisons. The experimental results are summarized in Table~\ref{tab:table5}.
\begin{table}[h]
\begin{center}
    \caption{Performance of three numerical optimization methods in terms of the objective function value, the prediction MSE, and the computational time}
    	\label{tab:table5}
    	\begin{tabular}{|c|c|c|}
    	\hline
    	Performance Metric & Electricity & Unemployment \\ \hline
    	GSM-GD Objective &8.330E+02 &3.838E+03 \\ \hline
    	GSM-MM Objective & 8.284E+02 &3.779E+03 \\ \hline
	GSM-ADMM Objective &8.266E+02 &3.776E+03 \\ \hline
	GSM-GD MSE &4.426E+03 &1.481E+04 \\ \hline	
	GSM-MM MSE &3.037E+03 &2.248E+03 \\ \hline
	GSM-ADMM MSE &2.220E+03 &2.222E+03 \\ \hline
	GSM-GD CT (s) &2272s &79189s \\ \hline	
	GSM-MM CT (s) &0.93s &8.40s \\ \hline
	GSM-ADMM CT (s) &6351.17s &160367.25s \\ \hline
    	\end{tabular}
\end{center}
\end{table}
%
%
%
Instead of striving to find a local minimum, we restrict the maximum number of iterations of the nonlinearly constrained ADMM due to its relatively slow convergence rate. Although the ADMM has not converged yet, it already found a weight estimate $\boldsymbol{\alpha}$ that leads to the smallest objective function value and prediction MSE among the selected numerical methods. However, the proposed nonlinearly constrained ADMM is less favorable than the MM method in terms of the computational time in both cases. 

As yet another comparison between the GSM-GD and GSM-ADMM, we showed the negative log-likelihood value versus iterations in Fig.~\ref{fig:convergencePlot}. It is clear from the results that the GSM-ADMM shows faster convergence rate as compared to GSM-GD. The gap of new introduced equality constraint is also depicted versus iterations in Fig.~\ref{fig:gapEquality}. 

Lastly, we give some guidance on the selection of a few key parameters of the proposed ADMM:
\begin{itemize}
\item Regularization parameter $\rho$: In general, a smaller $\rho$ leads to faster convergence rate of the method, while a larger $\rho$ leads to more stable convergence progress and smaller gap in the equality constraint but at a slower convergence rate. Good trade-off needs to be taken care of. 
%
\item Tolerance $\epsilon_{\boldsymbol{S}}$: in our ADMM, $\epsilon_{\boldsymbol{S}}$ is typically chosen small to waive the effect of the inexact solution of the sub-problem in Eq.(\ref{eq:ADMM-subproblem1}) and improve the overall convergence performance. However, a too small $\epsilon_{\boldsymbol{S}}$, on the other hand, is prohibited due to the high computational cost required for function evaluations. As rule-of-thumb, we could choose $\epsilon_{\boldsymbol{S}} \in \left[10^{-10},10^{-15}\right]$.
\end{itemize}

\section{Conclusion and Outlook}
\label{sec:Conclusion}
We studied automatic, optimal stationary kernel design with the good aim to let data choose the most appropriate kernel. We modified the SM kernel in the frequency domain by fixing the frequency and variance parameters to a big number of pre-selected grids. We conducted thorough studies on the properties of the resultant 1-D GSM kernel, including the sampling strategies of the grids, validity and low-rank property of all sub-kernels, and user-friendly initialization. The resultant GSM kernel demonstrates itself to be a linear multiple kernel. The ML based hyper-parameter optimization problem falls in difference-of-convex program and the solution is widely known to be sparse. Experimental results showed that the MM method achieved the best overall performance in various aspects, including convergence speed, economical computational time, insensitivity to an initial guess, competent fitting and prediction performance, etc. The fast computational speed of the MM method is obtained due to the low-rank properties of all GSM sub-kernels. On the other hand, the proposed ADMM showed great potential to achieve better local minimum but at the cost of larger computational time. Experimental results based on various classic time series data sets confirmed that the proposed 1-D GSM kernel is able to generate overall better performance than its 2-D counterpart and several other salient competitors of similar kind. Although the proposed 1-D GSM kernel showed outstanding prediction performance and very fast computational speed, it is more favorable to be used for low-dimensional time series. 

\section*{Acknowledgement} 
The author Feng Yin would like to thank Prof. Abdelhak M. Zoubir from Technische Universit\"{a}t Darmstadt and Prof. Xiaodong LI from UC Davis for the fruitful discussions on this manuscript during their visits at CUHK(SZ). 

\section*{Appendix}
\label{sec:Appendix}

\subsection{Proof of property (1): GSM kernel is a valid kernel}
\label{subsec:Appendix-A}
A necessary and sufficient condition for a function $k(\boldsymbol{x}, \boldsymbol{x}^{'})$ to be a valid kernel according to \cite{SC04} is that the corresponding kernel matrix, whose $(i,j)$-th entry is given by $k(\boldsymbol{x}_{i}, \boldsymbol{x}_{j})$, is PSD for all possible choices of $\boldsymbol{x} \in \mathcal{X}$.

For the proof, we need the following fundamental operations for constructing a new valid kernel $k(\boldsymbol{x}, \boldsymbol{x}^{'})$ that are well known from \cite{SC04} and \cite{Bishop06}: 
\begin{subequations}
\begin{align}
k(\boldsymbol{x}, \boldsymbol{x}') &= f(\boldsymbol{x}) f(\boldsymbol{x}') \label{eq:validKernel-property1} \\
k(\boldsymbol{x}, \boldsymbol{x}') &= c k_{1}(\boldsymbol{x}, \boldsymbol{x}')  \label{eq:validKernel-property2} \\
k(\boldsymbol{x}, \boldsymbol{x}') &= f(\boldsymbol{x}) k_{1}(\boldsymbol{x}, \boldsymbol{x}') f(\boldsymbol{x}')  \label{eq:validKernel-property3} \\
k(\boldsymbol{x}, \boldsymbol{x}') &= \exp \left( k_{1}(\boldsymbol{x}, \boldsymbol{x}')  \right) \label{eq:validKernel-property4} \\
k(\boldsymbol{x}, \boldsymbol{x}') &=  k_{1}(\boldsymbol{x}, \boldsymbol{x}') + k_{2}(\boldsymbol{x}, \boldsymbol{x}') \label{eq:validKernel-property5} \\
k(\boldsymbol{x}, \boldsymbol{x}') &=  k_{1}(\boldsymbol{x}, \boldsymbol{x}') \cdot k_{2}(\boldsymbol{x}, \boldsymbol{x}') \label{eq:validKernel-property6} 
\end{align}
\end{subequations}
where $k_{1}(\boldsymbol{x}, \boldsymbol{x}')$ and $k_{2}(\boldsymbol{x}, \boldsymbol{x}')$ are both known valid kernels; $f(\boldsymbol{x}): \mathbb{R}^d \rightarrow \mathbb{R}$ is any function; $c \geq 0$ is a constant. In our work, $\boldsymbol{x} = t$ and $d=1$. 

We will use the above results to prove that each sub-kernel function (omitting the subscript $i$) $k(t, t'; \sigma^2, \mu) = \exp \left[ -2 \pi^{2} (t-t')^2 \sigma^2 \right] \cos \left( 2 \pi (t-t') \mu \right)$ is a valid kernel. First, we let $k(t, t'; \sigma^2, \mu) = k_{1}(t, t'; \sigma^2) \cdot k_{2}(t, t'; \mu)$, where $k_{1}(t, t'; \sigma^2) \triangleq \exp \left[ -2 \pi^{2} (t-t')^2 \sigma^2 \right]$ and $k_{2}(t, t'; \mu) \triangleq \cos \left( 2 \pi (t-t') \mu \right)$. The first part $k_{1}(t, t'; \sigma^2)$ can be reformulated as
\begin{align}
k_{1}(t, t'; \sigma^2) &= \! \exp \!\! \left[ -2 \pi^{2} \sigma^2 t^2 \right] \exp \!\! \left[ 4 \pi^{2} \sigma^2 t t' \right] \exp \!\! \left[ -2 \pi^{2} \sigma^2 t^{\prime, 2} \right] \nonumber \\
&= \! f_{1}(t) \exp \!\! \left[ 4 \pi^{2} \sigma^2 \cdot k_{11}(t, t') \right] f_{1}(t'),
\end{align}
where $f_{1}(t) \triangleq \exp \!\! \left[ -2 \pi^{2} \sigma^2 t^2 \right]$ and $k_{11}(t, t') \triangleq t t'$ is the well known, valid linear kernel. Applying the fundamental operations given in Eq.(\ref{eq:validKernel-property2}), Eq.(\ref{eq:validKernel-property4}), and Eq.(\ref{eq:validKernel-property3}) in turn, yields a valid kernel $k_{1}(t, t'; \sigma^2)$. 

Next, we prove $k_{2}(t, t'; \mu) \triangleq \cos \left( 2 \pi (t-t') \mu \right)$ is also a valid kernel. This is done by reformulating the kernel as: 
\begin{align}
k_{2}(t, t'; \mu) &= \cos \left( 2 \pi \mu t - 2 \pi \mu t' \right) \nonumber \\
&= \cos \left( 2 \pi \mu t \right) \cos \left( 2 \pi \mu t' \right) + \sin \left( 2 \pi \mu t \right) \sin \left( 2 \pi \mu t' \right) \nonumber \\
&= f_{21}(t) f_{21}(t') + f_{22}(t) f_{22}(t'),
\end{align}
where $f_{21}(t) \triangleq \cos \left( 2 \pi \mu t \right)$ and $f_{22}(t) \triangleq \sin \left( 2 \pi \mu t \right)$. Applying the fundamental operations given in Eq.(\ref{eq:validKernel-property1}) and Eq.(\ref{eq:validKernel-property5}) in turn, yields a valid kernel $k_{2}(t, t'; \mu)$. 

Since both $k_{1}(t, t'; \sigma^2)$ and $k_{2}(t, t'; \mu)$ are valid kernels, according to Eq.(\ref{eq:validKernel-property6}), $k(t, t'; \sigma^2, \mu)$ is a valid kernel. The above proof holds generally for any $\sigma^2 \in \mathbb{R}_{+}$ and $\mu \in \mathbb{R}$. Consequently, each sub-kernel matrix $\boldsymbol{K}_{i}$ is a PSD matrix and so is $\boldsymbol{K} = \sum_{i=1}^{m} \alpha_{i} \boldsymbol{K}_{i}$. 

\subsection{Verification of property (5): low-rank property}
\label{subsec:Appendix-B}
We let $\boldsymbol{K}_i$ for a given grid with$(\mu_i, \sigma^2)$ be the $n \times n$ kernel matrix of the $i$-th sub-kernel $k_{i}(t, t')$ given in Eq.(\ref{eq:grid-SM-kernel-fixed-sigma}) with $t$ and $t'$ in $\{1,2,...,n\}$. The kernel matrix can be expressed in the form of Hadamard product as $\boldsymbol{K}_i = \boldsymbol{K}^{exp}_i \circ \boldsymbol{K}^{cos}_i$. Here, $\boldsymbol{K}^{exp}_i$ can be seen as the kernel matrix of its corresponding stationary kernel function $k^{exp}(\tau) \triangleq \exp(-2\pi^2 \tau^2 \sigma^2)$ and $\boldsymbol{K}^{cos}_i$ can be seen as the kernel matrix of its corresponding stationary kernel function $k_{i}^{cos}(\tau) \triangleq \cos(2 \pi \tau \mu_i)$. According to the rank inequality of Hadamard product of two matrices \cite{Styan73}, we have 
\begin{equation}
rank(\boldsymbol{K}_i) \leq rank(\boldsymbol{K}^{exp}) \cdot rank(\boldsymbol{K}^{cos}_i).
\end{equation}
We need the following two lemmas, (1) $rank(\boldsymbol{K}^{cos}_i) = 2$ for any grid $i$ and (2) $rank(\boldsymbol{K}^{exp}) \ll \frac{n}{2}$ for sufficiently small $\sigma^2$.

\begin{lemma}
For the kernel function $k_{i}^{cos}(\tau) = \cos(2 \pi \tau \mu_i)$ with any $\mu_i \in (0, 1/2)$, the rank of the corresponding kernel matrix is always equal to 2, i.e., $rank(\boldsymbol{K}^{cos}_i) = 2$ for any $i$.
\end{lemma}
\begin{proof}
The proof is as follows. It is obvious that first column of $\boldsymbol{K}^{cos}_i$, denoted by $\boldsymbol{k}_1 = [\cos(0x), \cos(1x),...,\cos((n-1)x)]^T$ and the second column, denoted by $\boldsymbol{k}_2 = [\cos(-1x), \cos(0x),\cos(1x),...,\cos((n-2)x)]^T$, where $x = 2\pi \mu_i$ is a constant in $(0, \pi)$ for any given $\mu_i$, are linearly independent. While from the 3rd column onward, each column can be expressed as a linear combination of the previous two columns simply because it holds for any $j \in \{-(n-1):1:(n-1)\}$, $\cos(jx) = \alpha \cos((j+2)x) + \beta \cos((j+1)x)$,
%
%
where $\alpha=-1$ and $\beta=\sin(2x)/\sin(x)$ are both irrespective of $j$. The derivation of $\alpha$ and $\beta$ is due to 
\begin{align}
\cos(jx) &= (\alpha \cos(2x) + \beta \cos(x))\cos(jx)  \nonumber \\
&- (\alpha \sin(2x) + \beta \sin(x)) \sin(jx).
\end{align}
Then, we let $\alpha \cos(2x) + \beta \cos(x) = 1$ and $\alpha \sin(2x) + \beta \sin(x) = 0$, then solve for $\alpha$ and $\beta$. The above steps prove that the kernel matrix $\boldsymbol{K}^{cos}_i$ is always of rank 2. 
\end{proof}


\begin{lemma}
For a time series with $n$ samples, i.e., $t=1,2,...,n$, when the variance parameter is selected to be $\sigma^2 \leq  \frac{2r+1}{2\pi^2(n-1)^2 \cdot C} \ll \frac{(n/2)+1}{2\pi^2(n-1)^2 \cdot C} \approx \frac{1}{4\pi^2(n-1)\cdot C}$, the rank of the kernel matrix $\boldsymbol{K}^{exp}$ corresponding to $k^{exp}(\tau) = \exp(-2\pi^2 \tau^2 \sigma^2)$, for any $\tau \in \{0,1,2,...,n-1\}$, satisfies $rank(\boldsymbol{K}^{exp}) \leq (2r+1) \ll n/2$ for some large constant number $C$. 
\end{lemma}

\begin{proof}
First of all, we show that there exists certain $K$ such that for each $\tau \in \{ 0, 1, 2, \ldots, n-1 \}$, the exponential function $\exp(-2\pi^2\sigma^2\tau^2)$, short for $\exp(a\tau^2)$, can be approximated by the first $K$ terms of its Taylor expansion, namely, 
\begin{math}
\exp(a\tau^2) = 1 + a\tau^2 + \frac{(a\tau^2)^2}{2!} + \frac{(a\tau^2)^3}{3!} + \ldots + \frac{(a\tau^2)^K}{K!} + R_{K+1},
\end{math}
where the remainder $R_{K+1} = \frac{(a\tau^2)^{(K+1)}}{(K+1)!}\exp(t \cdot a\tau^2)$ with $0<t<1$.
It is known that $\lim\limits_{K \to \infty} \left|R_{K+1}\right| \leq exp(|a\tau^2|) \cdot \lim\limits_{K \to \infty} \left|\frac{(a\tau^2)^{(K+1)}}{(K+1)!}\right| = 0$, hence for any $\epsilon > 0$, we can find a $K$ such that the approximation error $\left|R_{K+1}\right| < \epsilon$. 
%
In order to give a practical guidance on the selection of $\sigma$, we aim to find a number $K$ such that $\frac{(2\pi^2\sigma^2\tau^2)^K}{K!} > C \cdot \frac{(2\pi^2\sigma^2\tau^2)^{K+1}}{(K+1)!}$,
$\forall \tau \in \{0, 1, 2, \ldots, n-1 \}$, where $C$ is a large constant number. Conservatively for $\tau = n-1$, we have $\sigma^2 < \frac{K+1}{2\pi^2(n-1)^2 \cdot C}$, implying that for a fixed $n$, when $\sigma$ shrinks, the above inequality could still hold with smaller $K$. 
Since a drastic decrease in the absolute values of consecutive terms is our indicator for good approximation using Taylor expansion, in order to achieve $K = 2r \ll n/2$, we need to select $\sigma^2 \leq  \frac{2r+1}{2\pi^2(n-1)^2 \cdot C} \ll \frac{(n/2)+1}{2\pi^2(n-1)^2 \cdot C} \approx \frac{1}{4\pi^2(n-1)\cdot C}$. 

Next, we show that the rank of $\boldsymbol{K}^{exp}$ is at most $(2r+1)$ for a sufficiently small $\sigma^2$ due to the fact that $\exp(a\tau^2)$ can be well approximated by a linear combination of its $(2r+1)$ previous terms $\exp(a(\tau+1)^2), \exp(a(\tau+2)^2), \ldots, \exp(a(\tau+2r+1)^2)$ regardless of $\tau$. Our proof is as follows. For any give $\tau \in \{1,2,...,n-1\}$, we have the approximation
\begin{align}
\exp(a\tau^2) &\approx 1 + a\tau^2 + \frac{(a\tau^2)^2}{2!} + \frac{(a\tau^2)^3}{3!} + \ldots + \frac{(a\tau^2)^r}{r!} \nonumber \\
&= \tilde{a}_{0,1} + \tilde{a}_{0,1} \tau + \tilde{a}_{0,2} \tau^2 + \ldots +\tilde{a}_{0,2r} \tau^{2r},
\end{align}
where some of the coefficients are zeros. Similarly, for the $i$-th previous term we have 
\begin{math}
\exp(a(\tau+i)^2) \approx \tilde{a}_{i,0} + \tilde{a}_{i,1} \tau + \tilde{a}_{i,2} \tau^2 + \ldots +\tilde{a}_{i,2r} \tau^{2r}
\end{math}
for any $i \in \{1,2,\ldots,2r+1\}$. With the introduction of a coefficient matrix $\tilde{\boldsymbol{A}}_{(2r+1)\times(2r+1)}$ whose $ij$-th element is $\tilde{a}_{i,j-1}$, $\tilde{\boldsymbol{a}}_0 = [1,\tilde{a}_{0,1},\tilde{a}_{0,2},\ldots,\tilde{a}_{0,2r}]^T$ and $\boldsymbol{\beta} \in \mathbb{R}^{(2r+1)}$, we can construct a linear system $\tilde{\boldsymbol{A}} \boldsymbol{\beta} = \tilde{\boldsymbol{a}}_0$. Solving this linear system yields $\exp(a\tau^2) = [\exp(a(\tau+1)^2), \exp(a(\tau+2)^2), \ldots, \exp(a(\tau+2r+1)^2)] \boldsymbol{\beta}$. An important fact is that the solution of $\boldsymbol{\beta}$ has nothing to do with $\tau$, since both $\tilde{\boldsymbol{A}}$ and $\tilde{\boldsymbol{a}}_0$ are only in terms of the fixed $2\pi \sigma^2$. 
As a result, for any $j>2r+1$ and $2r \ll n/2$, the $j$-th column of $\boldsymbol{K}^{exp}$ can be reproduced by a linear combination of its $(2r+1)$ previous columns.

Combining the above two parts completes the proof of this theorem.
\end{proof}


\begin{corollary}
Following the above lemma, when $\sigma \rightarrow 0$, $rank(\boldsymbol{K}^{exp}) \rightarrow 1$. 
\end{corollary}


\subsection{Derivation of Eq.(\ref{eq:alpha_update})}
\label{subsec:Appendix-C}
Solving Eq.(\ref{eq:alpha_derivative}) for updated $\alpha_{i}$ yields
\begin{align}
\alpha_{i}^{k+1} &= \frac{-tr \left[  \left( \tilde{\boldsymbol{K}}_{-i} \boldsymbol{K}_i + \sigma_{e}^2 \boldsymbol{K}_{i} \right) \boldsymbol{S}^{k+1, T} \boldsymbol{S}^{k+1} - \boldsymbol{S}^{k+1} \boldsymbol{K}_{i} \right] }{tr\left( \boldsymbol{K}_{i}^{T} \boldsymbol{S}^{k+1, T} \boldsymbol{S}^{k+1} \boldsymbol{K}_{i} \right)} \nonumber \\
& \quad - \frac{tr \left( \boldsymbol{\Lambda}^{k, T} \boldsymbol{S}^{k+1} \boldsymbol{K}_{i} \right)}{\rho \cdot tr\left( \boldsymbol{K}_{i}^{T} \boldsymbol{S}^{k+1, T} \boldsymbol{S}^{k+1} \boldsymbol{K}_{i} \right)}.
\end{align}
It is noted that $\boldsymbol{K}_{i}^{T} = \boldsymbol{K}_{i}$ due to the symmetric property. Replace $\left( \tilde{\boldsymbol{K}}_{-i} \boldsymbol{K}_i + \sigma_{e}^2 \boldsymbol{K}_{i} \right)$ with $\left( \tilde{\boldsymbol{C}_{i}} - \alpha_{i}^{k} \boldsymbol{K}_{i} \right) \boldsymbol{K}_{i}^{T}$ in the above equation gives
\begin{align}
\alpha_{i}^{k+1} &= \frac{-tr \left[  \boldsymbol{K}_{i}^{T} \boldsymbol{S}^{k+1, T} \boldsymbol{S}^{k+1} \left( \tilde{\boldsymbol{C}_{i}} - \alpha_{i}^{k} \boldsymbol{K}_{i} \right)\right]}{tr\left( \boldsymbol{K}_{i}^{T} \boldsymbol{S}^{k+1, T} \boldsymbol{S}^{k+1} \boldsymbol{K}_{i} \right)} \nonumber \\
& \quad + \frac{tr \left( \boldsymbol{K}_{i}^{T} \boldsymbol{S}^{k+1, T} (\boldsymbol{I}_n - \frac{1}{\rho} \boldsymbol{\Lambda}^{k} )  \right) }{tr\left( \boldsymbol{K}_{i}^{T} \boldsymbol{S}^{k+1, T} \boldsymbol{S}^{k+1} \boldsymbol{K}_{i} \right)}.
\end{align}
Dragging $-\alpha_{i} \boldsymbol{K}_{i}$ outside of the first term, merging the other terms, and using the fact that both $\boldsymbol{K}_i$ and $\boldsymbol{S}$ are symmetric, yields Eq.~(\ref{eq:alpha_update}).
%
%
When $\boldsymbol{S}^{k+1}  \tilde{\boldsymbol{C}}_{i}$ is close to $\boldsymbol{I}_n$, the above update is approximately 
\begin{equation}
\alpha_{i}^{k+1} \approx \alpha_{i}^{k} - \frac{1}{\rho} \cdot \frac{tr \left[ \boldsymbol{K}_{i} \boldsymbol{S}^{k+1} \boldsymbol{\Lambda}^{k} \right] }{tr\left( \boldsymbol{K}_{i} \boldsymbol{S}^{k+1} \boldsymbol{S}^{k+1} \boldsymbol{K}_{i} \right)}.
\end{equation}


\bibliographystyle{IEEEtran}
\bibliography{refs}

\end{document}